\documentclass{aims}
\usepackage{amsmath}
\usepackage{tikz}
\pgfdeclarelayer{bg}    
\pgfsetlayers{bg,main}
\usepackage{subcaption}
\usepackage{paralist}
\usepackage[misc]{ifsym}
\usepackage{graphics} 
\usepackage{epsfig} %For pictures: screened artwork should be set up with an 85 or 100 line screen
\usepackage{graphicx}  
\usepackage{epstopdf}%This is to transfer .eps figure to .pdf figure; please compile your paper using PDFLeTex or PDFTeXify.
\usepackage[colorlinks=true]{hyperref}
\hypersetup{urlcolor=blue, citecolor=red}
\allowdisplaybreaks
\usepackage{esvect}
\usepackage[pagewise]{lineno}\nolinenumbers

\def\currentvolume{X}
 \def\currentissue{X}
  \def\currentyear{200X}
   \def\currentmonth{XX}
    \def\ppages{X--XX}
     \def\DOI{10.3934/xx.xxxxxxx}

\newtheorem{theorem}{Theorem}[section]

\newtheorem{lemma}[theorem]{Lemma}
\newtheorem{proposition}[theorem]{Proposition}

\theoremstyle{definition}
\newtheorem{definition}[theorem]{Definition}
\newtheorem{remark}[theorem]{Remark}

\DeclareMathOperator*{\argmin}{arg\,min}
\newtheorem{question}[theorem]{Question}

\title[SVM and Radon's Theorem]
      {Support vector machines and Radon's theorem}

\author[Henry Adams and Elin Farnell and Brittany Story]{}


\thanks{$^*$Corresponding author: Brittany Story}

\thanks{$^\bullet$ The foundational work for this paper was completed while Elin Farnell was a research scientist in the Department of Mathematics at Colorado State University.}
\thanks{$^\circ$ The foundational work for this paper was completed while Brittany Story was a graduate student in the Department of Mathematics at Colorado State University.}

\begin{document}
\maketitle

\centerline{\scshape
Henry Adams$^{{\href{mailto:henry.adams@colostate.edu}{\textrm{\Letter}}}1}$
and Elin Farnell$^{{\href{mailto:efarnell@amazon.com}{\textrm{\Letter}}}2,\bullet}$
and Brittany Story$^{{\href{mailto:bstory6@utk.edu}{\textrm{\Letter}}}3,\circ,*}$
}

\medskip

{\footnotesize
% Enter the full affiliation and country name:
% Do not insert commas or periods at the end of lines.
 \centerline{$^1$ Department of Mathematics}
  \centerline{Colorado State University}
  \centerline{Fort Collins, CO 80523, USA}
} % Do not forget to end {\footnotesize with the sign }

\medskip

{\footnotesize
 % Enter the full affiliation and country name:
  \centerline{$^2$ Amazon}
  \centerline{Seattle, WA, 98109}
}

\medskip

{\footnotesize
 % Enter the full affiliation and country name:
 \centerline{$^3$ Department of Mathematics}
  \centerline{University of Tennessee Knoxville}
  \centerline{Knoxville, TN 37996, USA}
}

\bigskip

% The name of the handling editor will be entered by AIMS production staff.
% "Communicated by Handling Editor" is not needed for special issue.
\centerline{(Communicated by Handling Editor)}

\begin{abstract}
A support vector machine (SVM) is an algorithm that finds a hyperplane which optimally separates labeled data points in $\mathbb{R}^n$ into positive and negative classes.
The data points on the margin of this separating hyperplane are called \textit{support vectors}.
We connect the possible configurations of support vectors to Radon's theorem, which provides guarantees for when a set of points can be divided into two classes (positive and negative) whose convex hulls intersect.
If the convex hulls of the positive and negative support vectors are projected onto a separating hyperplane, then the projections intersect if and only if the hyperplane is optimal.
Further, with a particular type of general position, we show that 
(a) the projected convex hulls of the support vectors intersect in exactly one point, 
(b) the support vectors are stable under perturbation, 
(c) there are at most $n+1$ support vectors, and 
(d) every number of support vectors from 2 up to $n+1$ is possible.
Finally, we perform computer simulations studying the expected number of support vectors, and their configurations, for randomly generated data.
We observe that as the distance between classes of points increases for this type of randomly generated data, configurations with fewer support vectors become more likely.
\end{abstract}

\maketitle

%%% Section %%%

\section{Introduction}

A support vector machine (SVM), when given a set of linearly separable points in $\mathbb{R}^n$, finds the separating hyperplane with the widest margin of separation between the two classes.
This distance is called the margin of error.
The vectors from the positive and negative classes that minimize the distance to the separating hyperplane are called the \textit{support vectors}; their positions define the location of the optimal separating hyperplane. 
SVMs have several different incarnations but this paper will focus on the most classical type of SVM: hard-margin SVM.
This type of SVM does not allow for any misclassified points, and it is restricted to linearly separable data.
For an example of hard-margin SVM in two dimensions, see Figure~\ref{fig:2dsvm}.
We are interested in the theoretical properties of support vectors in the context of hard-margin SVMs.
To describe these properties, we borrow ideas from geometry and topology, and in particular a result called \textit{Radon's theorem}.

Radon's theorem is a classical result in geometry, which states that given a set $T$ of at least $n+2$ points in Euclidean $n$-dimensional space $\mathbb{R}^n$, there are disjoint subsets $T_1$ and $T_2$ with $T=T_1\cup T_2$ so that the intersection of the convex hulls of $T_1$ and $T_2$ is non-empty.
Radon's theorem, along with other tools from convex geometry, can shed light on the properties of SVM support vectors.

In this paper, we explore the possible configurations of SVM support vectors when given a set of linearly separable labeled points.
Using Radon's theorem, we show that the projections of the convex hulls of the positive and negative support vectors onto the optimal separating hyperplane intersect.
We also show the converse result that given a separating hyperplane, if the projections of the convex hulls of the support vectors intersect, then the separating hyperplane is optimal.
There is more to say when the points are in strong general position (see Definition~\ref{def:sgp2}), which we show is a generic property.
When points are in strong general position in $\mathbb{R}^n$, we show that the number of support vectors is between $2$ and $n+1$.
In this strong general position setting, we show that the projections of the convex hulls of the support vectors onto the separating hyperplane intersect at a \emph{unique} point if and only if the separating hyperplane is optimal.
We also show that support vectors for points in strong general position are stable: perturbing the points by a sufficiently small amount preserves the points that are labeled as support vectors.
Finally, we provide computational experiments showing how the expected number of support vectors changes as a function of the distance between random linearly separable points.
As the distance between classes increases, configurations of points with fewer support vectors become dominant.

%%% Section %%%

\section{Background on support vector machines}
\label{sec:svmbackground}

\begin{figure}
\includegraphics[width=6cm]{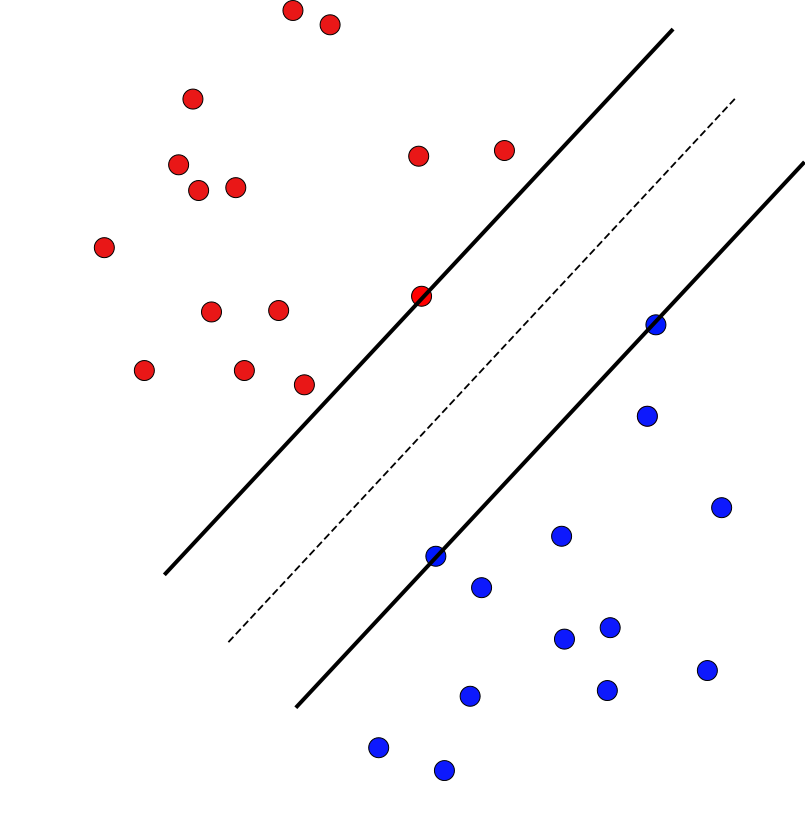}
\caption{Linearly separable two-class data, along with a linear classifier (a separating hyperplane) with the maximal margin of separation.}
\label{fig:2dsvm}
\end{figure}

Support vector machines are a popular supervised learning technique that have seen success in a wide variety of applications; see~\cite{Ma2014SVMapps} for several examples.
We first establish preliminaries on hyperplanes, convex hulls, and linear algebra.
Then we provide a brief introduction to support vector machines (SVMs) as background to this work, but we refer the interested reader to one of several resources for further information:~\cite{hofmann2008kernel,kecman2005support,muller2001introduction,smola2004tutorial}.

%%% Subsection %%%

\subsection{Preliminaries} 

For our work with support vector machines, our focus will be restricted to Euclidean space $\mathbb{R}^n$.
For a vector $\mathbf{x}\in\mathbb{R}^n$ we let $\|\mathbf{x}\|=\sqrt{x_1^2+\cdots+x_n^2}$ denote its length, and for vectors $\mathbf{x},\mathbf{y}\in\mathbb{R}^n$ we let $\langle \mathbf{x},\mathbf{y}\rangle=x_1y_1+\cdots+x_ny_n$ denote their inner product.
Note $\langle \mathbf{x},\mathbf{y}\rangle=\mathbf{x}^T\mathbf{y}$, where the superscript $T$ denotes the matrix transpose. Another name for the inner product is the dot product, and indeed we also denote $\langle \mathbf{x}, \mathbf{y}\rangle$ by $\mathbf{x}\cdot \mathbf{y}$.

A \textit{linear subspace} of $\mathbb{R}^n$ is a vector space subset of $\mathbb{R}^n$; any linear subspace $V\subseteq\mathbb{R}^n$ can be written as $V=\left\{\sum_{i=1}^k c_i\mathbf{v}_i~|~c_i\in \mathbb{R}\right\}$ for some collection of vectors $\mathbf{v}_1,\ldots,\mathbf{v}_k\in \mathbb{R}^n$ with $1\le k\le n$.
The subspace has dimension $k$ if and only if the vectors are linearly independent.
An \textit{affine subspace} of $\mathbb{R}^n$ is a translation of a linear subspace, such as $\mathbf{x}+V:=\{\mathbf{x}+\mathbf{v}~|~\mathbf{v}\in V\}$, where $V$ is a linear subspace of $\mathbb{R}^n$.
A $k$-dimensional affine subspace is also called a \textit{$k$-flat}.
The \emph{affine span} of a set of vectors in $\mathbb{R}^n$ is the smallest affine subspace that contains that set.

Any linear subspace of dimension $n-1$ in $\mathbb{R}^n$ can be written as $\{\mathbf{x}\in\mathbb{R}^n~|~\mathbf{w}^T\mathbf{x}=0\}$ for some normal vector $\mathbf{w}\in\mathbb{R}^n$.
Similarly, any affine subspace of dimension $n-1$ in $\mathbb{R}^n$ can be written as $\{\mathbf{x}\in\mathbb{R}^n~|~\mathbf{w}^T\mathbf{x}=b\}$ for some normal vector $\mathbf{w}\in\mathbb{R}^n$ and offset $b\in \mathbb{R}$.
We refer to $(n-1)$-dimensional affine subsets of $\mathbb{R}^n$ as \textit{hyperplanes}.

As Radon's theorem addresses intersections of convex hulls, we first must establish what it means to be a convex hull. 
A set $T\subseteq \mathbb{R}^n$ is said to be \textit{convex} if for all $\mathbf{x},\mathbf{y}\in T$ and $t\in (0,1)$, we also have that $t\mathbf{x}+(1-t)\mathbf{y}\in T$. 
Intuitively, in two-dimensional space, the convex hull of a set of points can be visualized by putting a rubber band around all of the points in the set. 
More formally, let $\mathrm{conv}(T)$ denote the \textit{convex hull} of a set of points $T\subseteq \mathbb{R}^n$, which is the minimal convex set in $\mathbb{R}^n$ containing $T$,
or equivalently, the intersection of all convex sets in $\mathbb{R}^n$ containing $T$. 
More explicitly, we can define $\mathrm{conv}(T)$ as 
\[\mathrm{conv}(T)=\left\{\sum_{i=1}^k \alpha_{i}\mathbf{x}_{i}\,{\Bigg |}\,k\ge 1,\,\mathbf{x}_{i}\in T,\,\alpha_{i}\ge 0,\,\sum _{i=1}^{k}\alpha _{i}=1\right\}.\]
Finally, let a \textit{convex combination} of a set of points be a linear combination of those points where all coefficients are non-negative and sum to 1.

%%% Subsection %%%

\subsection{Relevant linear algebra}\label{ssec:lin_alg}

We recall some facts from linear algebra, which will be used in Section~\ref{sec:svm-general-selected} to show that the SVM support vectors remain support vectors under small perturbations.

\begin{lemma}
\label{lem:lin_dep}
A set of vectors $\mathbf{a}_1, \mathbf{a}_2, \ldots \mathbf{a}_k \in \mathbb{R}^n$ is linearly dependent if and only if $\det(A^{T} A)=0$, where $A$ is the $n\times k$ matrix whose $i$-th column is $\mathbf{a}_i$. 
\end{lemma}

\begin{proof}
We proceed with a sequence of if-and-only-if equivalences.
First, the vectors $\left\{\mathbf{a}_i\right\}_{i=1}^{i=k}$ are dependent if and only if the nullspace of $A$ is non-trivial.
But the nullspace of $A$ is non-trivial if and only if the nullspace of $A^T A$ is non-trivial: \\
($\Longrightarrow$) If the nullspace of $A$ is non-trivial, then there exists a non-zero vector $\mathbf{x}$ such that $A\mathbf{x}=\mathbf{0}.$
But then we also have $(A^TA)\mathbf{x}=A^T(A\mathbf{x})=A^T\mathbf{0}=\mathbf{0}$, so the nullspace of $A^TA$ is non-trivial as well.\\
($\Longleftarrow$) Suppose there is a non-zero vector $\mathbf{x}$ such that $A^T A\mathbf{x}=\mathbf{0}$.
Then $\mathbf{0} = \langle \mathbf{x}, A^T A\mathbf{x}\rangle = \langle A\mathbf{x}, A\mathbf{x}\rangle = \| A\mathbf{x} \|^2,$ and hence $A\mathbf{x}=0$.\\
So we have that the vectors $\left\{\mathbf{a}_i\right\}_{i=1}^{i=k}$ are dependent if and only if the nullspace of $A$ is non-trivial if and only if the nullspace of $A^TA$ is non-trivial, which occurs precisely when the determinant of $A^T A$ is zero.
\end{proof}

The following lemma implies that any affine subspace of $\mathbb{R}^n$ has a unique closest point to the origin.

\begin{lemma}
\label{lem:min_Euclidean_norm}
Let $\mathbf{a}_1, \mathbf{a}_2, \ldots \mathbf{a}_k \in \mathbb{R}^n$ be linearly independent and let $\mathbf{d}\in \mathbb{R}^n$.
Then, there is a unique choice of coefficients, $\mathbf{t}=(t_1,\ldots,t_k)^T$, minimizing the Euclidean norm of the vector 
\[\mathbf{w}=\mathbf{d}+\sum\limits_{i=1}^k t_i \mathbf{a}_i.\]
Specifically, 
\[\mathbf{t} = -(A^T A)^{-1}A^T \mathbf{d}\text{ and }\mathbf{w} = \mathbf{d} - A(A^T A)^{-1}A^T \mathbf{d},\]
where $A$ is the $n\times k$ matrix whose $i$-th column is $\mathbf{a}_i$.
It follows that the entries of $\mathbf{t}$ and $\mathbf{w}$ are rational functions of the entries of the vectors $\mathbf{a}_1, \mathbf{a}_2, \ldots \mathbf{a}_k, \mathbf{d}$, where the denominators may be taken to be $\det(A^T A)$.
\end{lemma}

\begin{proof}
Note, the minimum norm occurs exactly when $\mathbf{w}$ is orthogonal to each $\mathbf{a}_i$. 
This implies that for all $1\le i \le k$,
\[0=\langle\mathbf{a}_i,\mathbf{w}\rangle = \langle\mathbf{a}_i,\mathbf{d}+A\mathbf{t}\rangle = \langle\mathbf{a}_i,\mathbf{d}\rangle + \langle\mathbf{a}_i,A\mathbf{t}\rangle = \mathbf{a}_i^T\mathbf{d}+\mathbf{a}_i^T A\mathbf{t}.\]
Setting these $k$ equations on top of each other gives the vector equation $\mathbf{0}=A^T \mathbf{d}+A^T A \mathbf{t}$. 
Hence, $\mathbf{t} = -(A^T A)^{-1}A^T \mathbf{d}$ and $\mathbf{w} = \mathbf{d}+A\mathbf{t} = \mathbf{d} - A(A^T A)^{-1}A^T \mathbf{d}$ is the unique solution, since $A^T A$ is invertible by Lemma~\ref{lem:lin_dep}.
\end{proof}

%%% Subsection %%%

\subsection{Optimal separating hyperplanes} \label{ssec:opt-sep}

Consider a finite set $X\subseteq \mathbb{R}^n$ such that each point $\mathbf{x}\in X$ is equipped with a label, $1$ or $-1$.
Let $X_+\subseteq X$ be the subset of $X$ that has the label $1$, and similarly for $X_-$.
We assume the two classes are linearly separable; that is, we assume there exists a hyperplane $H$ (an $(n-1)$-dimensional affine subspace of $\mathbb{R}^n$) such that each point in $X_+$ is on one side of the hyperplane, and each point in $X_-$ is on the opposite side of the hyperplane.
See Figure~\ref{fig:2dsvm} for an example of linearly separable two-class data.

Let us explain what we mean when we say that a separating hyperplane $H$ is optimal.
Let $H_+$ denote the hyperplane that is a translation of $H$ and that satisfies both of the following conditions:
\begin{itemize}
    \item $H_+\cap X_+\neq \emptyset$, and 
    \item $H_+$ separates $X_-$ from the remaining points in $X_+$.
\end{itemize}
Define $H_-$ similarly.
The margin of separation is the distance between the parallel hyperplanes $H_+$ and $H_-$. 
A separating hyperplane $H$ is \emph{optimal} (or $H$, $H_+$, and $H_-$ are \emph{optimal}) if $H_+$ and $H_-$ have the largest possible margin of separation, amongst all such pairs of parallel hyperplanes.
In general, once a normal vector to a separating hyperplane is identified, it is straightforward to find $H_+$, $H_-$, and $H$ via translation (typically, $H$ is chosen to be halfway in-between $H_+$ and $H_-$).
As a result, we will consider the identification of a normal vector to the optimal separating hyperplane as equivalent to identifying the optimal hyperplane $H$ itself, and we will refer to the (implicitly defined) hyperplane $H$ as \emph{optimal}.
The problem of finding the optimal separating hyperplane is precisely the optimization problem that the SVM algorithm solves, as will be discussed in the following section.

%%% Subsection %%%

\subsection{SVMs in the linearly separable case} \label{ssec:svm_lin_sep}

We begin with the most basic definition of support vector machines, and discuss more general versions in Section~\ref{ssec:soft-margin}. 
What follows is a common derivation of the SVM problem; see for example~\cite{Cervantes2020svm_survey}.
Consider a set of training data $X=\left\{ \mathbf{x}_i\right\}_{i=1}^m$ and associated class labels $\left\{y_i\right\}_{i=1}^m$, where each $y_i\in \left\{-1,1\right\}$.
We assume the two classes $X_+$ and $X_-$ are linearly separable.
From this data, we train a support vector machine, which is a linear classifier maximizing the separation between the two classes of training data.

We construct the support vector machine as a nonlinear optimization problem with constraints, as follows.
First, define the classifier as $f(\mathbf{x})=\mathbf{w}^T\mathbf{x}+b$, where $\mathbf{w}$ and $b$ are the parameters that must be learned.
Then the hyperplane $f(\mathbf{x})=0$ defines a linear decision boundary, and the sign of $f(\mathbf{x})$ (denoted $\mathrm{sign}(f(\mathbf{x}))\in\{-1,1\}$) determines which of the two classes the classifier predicts as the true membership class of data point $\mathbf{x}$.

Note that for any constant $c\neq 0$, the parameters $c\mathbf{w}$ and $cb$ would define the same hyperplane as $\mathbf{w}$ and $b$.
Thus, we introduce the notion of a canonical hyperplane as in~\cite{kecman2005support}: Given a set of data $\mathbf{x}_i\in X$, a hyperplane representation is \textit{canonical} if $\min_{\mathbf{x}_i}|\mathbf{w}^T\mathbf{x}_i+b|=1$.
We define the \textit{support vectors} for such a hyperplane to be precisely those $\mathbf{x}_i$ for which this minimum is achieved. 

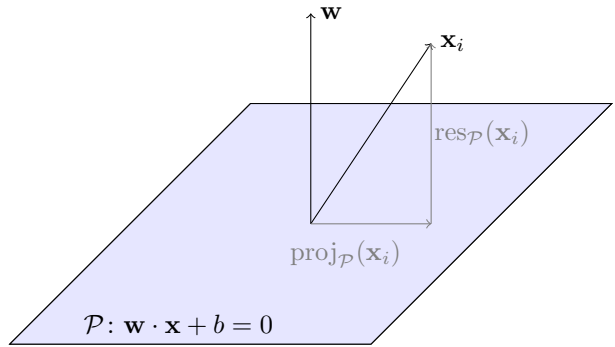
\begin{figure}[htb]
\begin{tikzpicture}[scale=0.8]
\begin{pgfonlayer}{bg}
    \coordinate [] (A) at (-5,-2);
    \coordinate [] (B) at (-1,2);
    \coordinate [] (C) at (5,2);
    \coordinate [] (D) at (1,-2);
    \foreach \i in {A,B,C,D}
    \draw[-, fill=blue!10, opacity=.7] (A)--(B)--(C)--(D)--cycle;
\end{pgfonlayer}
    \coordinate [] (E) at (0,0);
    \coordinate [label={[right]$\mathbf{w}$}] (F) at (0,3.5);
    \coordinate [] (G) at (2,0);
    \coordinate [label={[right] $\mathbf{x}_i$}] (H) at (2,3);
    \coordinate [label={[right] \textcolor{gray}{$\textrm{res}_{\mathcal{P}}(\mathbf{x}_i)$}}] (I) at (1.9,1.5);
    \coordinate [label={[right] \textcolor{gray}{$\mathrm{proj}_\mathcal{P}(\mathbf{x}_i)$}}] (J) at (-0.5,-.5);
    \coordinate [label={$ \mathcal{P}\colon \mathbf{w}\cdot \mathbf{x}+b=0$}] (K) at (-2.2,-2);
   \draw [->] (E) -- (F); 
   \draw [->] (E) -- (H);
   \draw [gray,->] (E) -- (G); 
   \draw [gray,->] (G) -- (H);
\end{tikzpicture}
\caption{Decision boundary hyperplane $\mathcal{P}$ along with a vector $\mathbf{x}_i$ and its decomposition with respect to $\mathcal{P}$, where $\mathbf{w}$ and $\mathbf{x}_i$ are drawn with tails at the same base point in the hyperplane for context.}
\label{fig:margin}
\end{figure}

Recall that the support vector machine defined by $f(\mathbf{x})=\mathbf{w}^T\mathbf{x}+b$ should be a linear classifier with maximal margin of separation between two data classes. 
Therefore, by characterizing the margin of separation between the classes, we can determine the form the appropriate optimization problem should take.
Consider the hyperplane $\mathcal{P}$ defined by
\[\mathcal{P}=\{\mathbf{x}\in\mathbb{R}^n~|~\mathbf{w}^T\mathbf{x}+b=0\},\]
and let $\mathbf{x}_i$ be a support vector with angle less than $\pi/2$ with $\mathbf{w}$ (thus, $\mathbf{x}_i$ is on the side of the hyperplane that $\mathbf{w}$ points to).
Since our hyperplane representation is canonical, this gives $\mathbf{w}\cdot\mathbf{x}_i=1-b$.
As in Figure~\ref{fig:margin}, note that $\mathbf{x}_i$ can be decomposed orthogonally into a component in $\mathcal{P}$ and its residual: $\mathbf{x}_i=\mathrm{proj}_\mathcal{P}(\mathbf{x}_i)+\textrm{res}_\mathcal{P}(\mathbf{x}_i)$.
The minimum distance $d$ from the hyperplane $\mathcal{P}$ to $X$ is then $d=\|\textrm{res}_\mathcal{P}(\mathbf{x}_i)\|$.
Note that by construction, we have $\textrm{res}_\mathcal{P}(\mathbf{x}_i)=d\frac{\mathbf{w}}{\|\mathbf{w}\|}$, for $d\in\mathbb{R}$, where $\frac{\mathbf{w}}{\|\mathbf{w}\|}$ is the unit vector in the direction of $\mathbf{w}$.
Taking the dot product of the decomposition of $\mathbf{x}_i$ with $\mathbf{w}$ and using the fact that $\mathbf{x}_i$ is a support vector, we have 
\begin{align*}
1-b&=\mathbf{w}\cdot\mathbf{x}_i\\
&=\mathbf{w}\cdot\mathrm{proj}_\mathcal{P}(\mathbf{x}_i)+\mathbf{w}\cdot\textrm{res}_\mathcal{P}(\mathbf{x}_i)\\
&=-b + \mathbf{w}\cdot\left(d\frac{\mathbf{w}}{\|\mathbf{w}\|}\right)&&\text{since }\mathrm{proj}_\mathcal{P}(\mathbf{x}_i)\in\mathcal{P}\\
&=-b + d\|\mathbf{w}\|.
\end{align*}
Adding $b$ to both sides gives $1= d\|\mathbf{w}\|$, i.e., $d=\frac{1}{\|\mathbf{w}\|}$.

A similar argument produces the same conclusion when $\mathbf{x}_i$ is assumed to be a support vector with angle $\theta\in\left(\pi/2,\pi\right]$ with $\mathbf{w}$.
Thus, we have that the minimum distance from the linear decision boundary to any point $\mathbf{x}_i$ in $X$ is $\frac{1}{\|\mathbf{w}\|}$.
This means that the margin of linear separation between the classes is $\frac{2}{\|\mathbf{w}\|}$.
Consequently, in order to maximize the margin of separation, we seek to minimize the norm of $\mathbf{w}$.
For convenience, we will instead minimize $\frac{1}{2}\|\mathbf{w}\|^2$.

To characterize the correct classification of all $\mathbf{x}_i\in X$, we observe that if $y_i=1$, a support vector machine with perfect classification must have $f(\mathbf{x}_i)>0$, and if $y_i=-1$, we require $f(\mathbf{x}_i)<0$.
But our definition of canonical hyperplane requires that $|f(\mathbf{x}_i)|\geq 1$ for all $i$.
So classification of all points is correct if $y_if(\mathbf{x}_i)\geq 1$ for all $i$, i.e.\ if $ y_i\left(\mathbf{w}^T\mathbf{x}_i+b \right)\geq 1$ for all $i$.
Therefore, the support vector machine optimization problem is
\[\argmin_{\mathbf{w},b} \frac{1}{2}\|\mathbf{w}\|^2 \textrm{ subject to } y_i\left(\mathbf{w}^T\mathbf{x}_i+b \right)\geq 1 \textrm{ for all }i.\]

%%%%%%%%% Subsection %%%%%%%%%%%%%

\subsection{Solving the SVM optimization problem}
\label{ssec:solving}

Now that we have an optimization problem, we need some additional tools to solve it.
The Karush-Kuhn-Tucker (KKT) theorem can be applied to the SVM optimization problem and provide an optimal solution, provided the original problem satisfies certain conditions, which can be found at~\cite{platt1998svmsmo}.
The KKT theorem works like Lagrange multipliers:
under the right conditions, it takes a lower-dimensional bounded problem and turns it into a higher-dimensional unbounded problem, which can be solved using calculus.
Upon applying the KKT conditions to SVMs, the resulting unbounded optimization problem is 
\[
L(\alpha)=\sum_{j=1}^m\alpha_j-\frac{1}{2}\sum_{i=1}^m\sum_{j=1}^m\alpha_i\alpha_jy_iy_j\langle \mathbf{x}_i,\mathbf{x}_j\rangle,
\]
where the scalars $\alpha_j$ play the role of Lagrange multipliers.

From here, we can solve the unbounded optimization problem using a variety of methods.
In the dual problem, the non-zero numbers $\alpha_j$ correspond to the vectors $\mathbf{x}_j$ that are support vectors.
This defines the desired hyperplane which best divides the two classes of data.
Indeed (see for example~\cite[Section~7.1]{smola2008introduction}), we have $\mathbf{w}=\sum_j y_j\alpha_j\mathbf{x}_j$ and $\sum_j \alpha_j y_j=0$, where $\mathbf{w}$ is the normal vector to the optimal separating hyperplane.

%%%%%%%%%%%% Subsection %%%%%%%%%%%%

\subsection{SVMs in the general case (soft margins)}
\label{ssec:soft-margin}

The most common versions of support vector machines do not require the hypothesis that the data classes are linearly separable --- indeed, the optimization problem is edited to optimize over two competing preferences: maximizing the margin of separation, i.e.\ the ``width of the road'', while also minimizing the extent to which points lie within the margin or are misclassified. 
To accomplish this, we look at the hinge loss function, $\max(0,1-y_i(\mathbf{w}\cdot\mathbf{x}_i+b))$, which yields 0 if the point is on the correct side of the margin, and which provides a loss proportional to how far away a point on the incorrect side is from the margin.
As discussed in~\cite{smola2004tutorial}, the soft-margin optimization problem is to minimize
\[
\left[\frac{1}{m}\sum_{i=1}^m \max(0,1-y_i(\mathbf{w}\cdot\mathbf{x}_i+b))\right]+\lambda\|\mathbf{w}\|^2,
\]
where $\lambda$ is a parameter which tunes for the compromise between the width of the margin and classification accuracy on the training data.

Given this new optimization problem, we could ask what support vector configurations are possible in the soft-margin case, and whether or not those configurations are related to Radon's theorem.
Further, we could compare the soft-margin support configurations to those found in the hard-margin case.
This paper will focus on hard-margin support vector machines, restricted to linearly separable data, which is the mathematically simpler case.

%%%%%%%%%%%%% Subsection %%%%%%%%%%%%%%%

\subsection{SVMs and VC dimension} 

The Vapnik--Chervonenkis (VC) dimension, first introduced by Vapnik and Chervonenkis in 1971~\cite{vapnik1971convergance}, is a measure of the complexity of a classification model.
Given a finite set of points $X$, note that there are $2^{|X|}$ ways to label the points with labels $1$ or $-1$.
We say that a classification model can \textit{shatter} a set of points $X$ if no matter how the labels $1$ or $-1$ are assigned to $X$, we can find a classifier from that model that correctly recovers the assigned labels.
The \emph{VC dimension} of a classification model is the cardinality of the largest set of points that model can shatter~\cite{sakr2016vc}.

As an example, consider the VC dimension of affine separators (binary classifiers that are defined by location on one side or the other of an affine hyperplane) in $\mathbb{R}^2$.
Note that support vector machines are a particular way to choose the affine hyperplane for an affine separator.
Consider first a set $X\subseteq \mathbb{R}^2$ of three points ($|X|=3$) that do not lie on a single line.
Note there are $2^3=8$ ways to label these points with values in $\{-1,1\}$.
No matter how those three points are labelled with values in $\{-1,1\}$, there exists an affine line that correctly separates the two classes ($1$ and $-1$).
Thus, the VC dimension of affine separators in $\mathbb{R}^2$ is at least three.

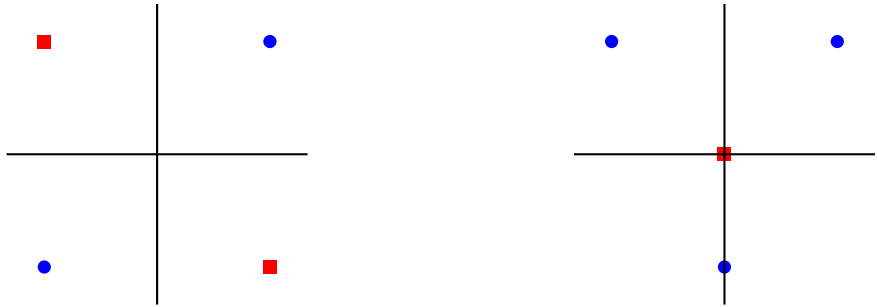
\begin{figure}[htb]
\begin{center}
\begin{tabular}{c p{8em} c}
\begin{tikzpicture}[scale=0.5]
    \node [circle,fill=blue,inner sep=0pt,minimum size=5pt] (A) at (3,3) {};
    \node [rectangle,fill=red,inner sep=0pt,minimum size=5pt] (B) at (-3,3) {};
    \node [circle,fill=blue,inner sep=0pt,minimum size=5pt] (C) at (-3,-3) {};
    \node [rectangle,fill=red,inner sep=0pt,minimum size=5pt] (D) at (3,-3) {};
    \draw[-,thick] (0,-4)--(0,4);
    \draw[-,thick] (-4,0)--(4,0);
\end{tikzpicture} 
&&
\begin{tikzpicture}[scale=0.5]
    \node [circle,fill=blue,inner sep=0pt,minimum size=5pt] (A) at (3,3) {};
    \node [circle,fill=blue,inner sep=0pt,minimum size=5pt] (B) at (-3,3) {};
    \node [circle,fill=blue,inner sep=0pt,minimum size=5pt] (C) at (0,-3) {};
    \node [rectangle,fill=red,inner sep=0pt,minimum size=5pt] (D) at (0,0) {};
    \draw[-,thick] (0,-4)--(0,4);
    \draw[-,thick] (-4,0)--(4,0);
\end{tikzpicture}\\
\end{tabular}
\end{center}
\caption{
Examples of two sets in $\mathbb{R}^2$ of size 4 that cannot be shattered by affine separators.
Consider the labeled classes drawn as red squares and blue circles.
In each example, the convex hulls of the two classes intersect at the origin.
This implies that no affine separator can correctly classify according to these labels.
}
\label{fig:unshatterable}
\end{figure}

We now argue that the VC dimension of affine separators in $\mathbb{R}^2$ is exactly three~\cite{vapnik2013nature}.
% page 81
To do this, we must show that for any set $X\subseteq \mathbb{R}^2$ with $|X|=4$, the set $X$ cannot be shattered.
Radon's theorem, which is discussed further in Section~\ref{sec:radon}, states that since $X$ is a set of $4$ points in $\mathbb{R}^2$, there must be disjoint sets $X_+$ and $X_-$ with $X=X_+\cup X_-$ and $\mathrm{conv}(X_+)\cap \mathrm{conv}(X_-)\neq\emptyset$ (see Figure~\ref{fig:unshatterable}).
If we label all the points in $X_+$ as belonging to the positive class, and all the points in $X_-$ as belonging to the negative class, then no affine separator will be able to correctly classify these labels.
Indeed, any affine separator assigning each point in $X_+$ the label $1$ must assign the entire convex hull $\mathrm{conv}(X_+)$ the label $1$, and any affine separator assigning each point in $X_-$ the label $-1$ must assign all of $\mathrm{conv}(X_-)$ the label $-1$.
This contradicts the fact that $\mathrm{conv}(X_+)\cap \mathrm{conv}(X_-)$ is non-empty!
Hence no set of four points in $\mathbb{R}^2$ can be shattered by an affine separator, and so the VC dimension of affine separators in $\mathbb{R}^2$ is three.

More generally, Radon's theorem states that if $X$ is a set of $k$ points in Euclidean $n$-dimensional space $\mathbb{R}^n$ with $k\geq n+2$, then there are disjoint sets $X_+$ and $X_-$ with $X=X_+\cup X_-$ and $\mathrm{conv}(X_+)\cap \mathrm{conv}(X_-)\neq\emptyset$.
So, if we label the points in $X_+$ as the positive class and the points in $X_-$ as the negative class, then since their convex hulls intersect, this set of points cannot be shattered by an affine separator.
Thus, no configuration of $n+2$ or more points in $\mathbb{R}^n$ can be shattered.
It follows that $n+1$ is an upper bound for the VC dimension of affine separators in $\mathbb{R}^n$.
It is also true that so long as $n+1$ points in $\mathbb{R}^n$ do not lie in an $(n-1)$-dimensional affine plane (for example, if those points live at the vertices of a regular $n$-simplex), then those points can be shattered by an affine separator.
So the VC dimension of affine separaters in $\mathbb{R}^n$ is exactly $n+1$.

%%%%%%%%%%%% Section %%%%%%%%%%%%%%%

\section{Background on Radon's theorem}
\label{sec:radon}

Radon's theorem is a classical result in convex geometry that has applications across a variety of fields.
We proceed with a self-contained description of Radon's theorem, and we include corollaries that pertain to classifying support vectors in $\mathbb{R}^n$.
First we state and prove Radon's theorem.
Although the proof is well-known, we include it here since the ideas within will reappear later.
The original reference is~\cite{radon1921mengen}, and modern references include, for example,~\cite{clarkson1996approximating,Matousek:2002:LDG,peterson1972geometry}.
The proof we give follows~\cite{Matousek:2002:LDG}.
%Additional proofs of Radon's theorem can be found in a variety of sources; for example a geometric proof can be found in a paper by Peterson ~\cite{peterson1972geometry}.
% algebraic proof: (\footnote{\url{https://perso.esiee.fr/~mustafan/TechnicalWritings/math-lec2.pdf}}) 

\begin{theorem}[Radon's Theorem]
\label{thm:Radon}
If $T$ is a set of $k$ points in $\mathbb{R}^n$ with $k\geq n+2$, then there are disjoint sets $T_1$ and $T_2$ with $T=T_1\cup T_2$ and $\mathrm{conv}(T_1)\cap \mathrm{conv}(T_2)\neq\emptyset$.
\end{theorem}

\begin{proof}
It suffices to prove the case $k=n+2$.
Let $T=\{\mathbf{x}_0, \mathbf{x}_1, \ldots, \mathbf{x}_{n+1}\}$ be a set of $n+2$ points in $\mathbb{R}^n$.
There exist coefficients, $a_0$, \ldots, $a_{n+1}$, not all zero, such that 
\begin{equation}\label{eq:radon}
\sum\limits_{i=0}^{n+1} a_i \mathbf{x}_i=\mathbf{0}\quad\mbox{and}\quad\sum\limits_{i=0}^{n+1} a_i=0.
\end{equation} 
Indeed, this is because \eqref{eq:radon} is a collection of $n+1$ homogeneous linear equations with $n+2$ unknowns.
Consider a specific non-trivial solution, $(a_0,\ldots,a_{n+1})$.
Partition the set $S=\{0,\ldots,n+1\}$ so that $S_1$ contains all $i$ with $a_i>0$, and $S_2$ contains all $i$ with $a_i\le 0$; both $S_1$ and $S_2$ are necessarily non-empty.
By \eqref{eq:radon}, there exists a point $\mathbf{v}\in\mathbb{R}^n$ such that 
\[\mathbf{v}=\sum_{i\in S_1}\frac{a_i}{c}\mathbf{x}_i=-\sum_{j\in S_2}\frac{a_j}{c}\mathbf{x}_j,\] 
where $c=\sum\limits_{i\in S_1} a_i=-\sum\limits_{j\in S_2} a_j$.
Let $T_1=\{\mathbf{x}_i\in T~|~i\in S_1\}$, and similarly for $T_2$ (from which it follows that $T=T_1\cup T_2$).
Thus $\mathbf{v}$ is an intersection point of the two convex hulls, since each sum is a representation of $\mathbf{v}$ as a convex combination of points in $T_1$ and $T_2$, respectively.
Therefore, $\mathrm{conv}(T_1)\cap \mathrm{conv}(T_2)\neq\emptyset$.
\end{proof}

In other words, given a set of at least $n+2$ points in $\mathbb{R}^n$, there exists a partition into two parts such that the convex hulls of the two parts intersect.
The points in the intersection of the convex hulls will be relevant throughout the rest of the paper. 
We call them \textit{Radon points} since they are guaranteed by Radon's theorem.
We refer to labeled configurations of points that have a Radon point as \textit{Radon configurations}; see Figure~\ref{fig:TwoRadonR2}.

\begin{definition}
Given a finite set $T\subseteq \mathbb{R}^n$ with disjoint labeled subsets $T_1$ and $T_2$, a \textit{Radon point} is any point $\mathbf{v}\in \mathrm{conv}(T_1)\cap\mathrm{conv}(T_2)$.
We refer to the labeled subsets $T_1$ and $T_2$ for which there is a Radon point as a \textit{Radon configuration}.
\end{definition}

Radon's theorem is a blend of convex geometry and topology.
A topological version of Radon's theorem states that if $f\colon \Delta^{n+1}\to \mathbb{R}^n$ is a continuous map from the $(n+1)$-simplex to $n$-dimensional Euclidean space, then there are two disjoint simplices of $\Delta^{n+1}$ whose images under $f$ intersect; see the paper by Bajm{\'o}czy and B{\'a}r{\'a}ny \cite{Bajmczy1979OnAC}.
We recover Theorem~\ref{thm:Radon} in the case $k=n+2$ by letting $f$ be an affine map sending the $n+2$ vertices of the $(n+1)$-simplex $\Delta^{n+1}$ to the $n+2$ points in $T$.
Radon's theorem is made even more topological by variants of Tverberg's theorem; see Question~\ref{ques:tverberg}.

Some applicable corollaries follow from the proof of Radon's theorem.
For this section, we use the following definition of general position.

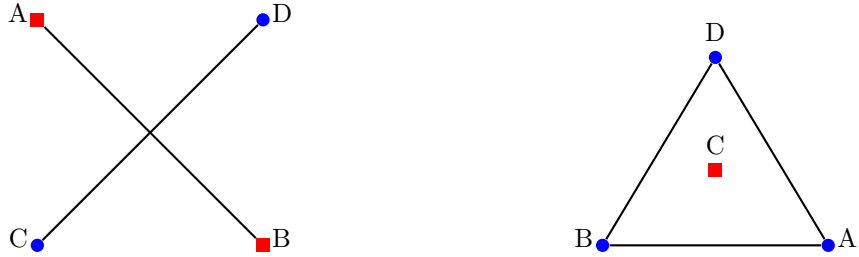
\begin{figure}
\begin{center}  
\begin{tabular}{ c p{8em} c }
\begin{tikzpicture}[scale=0.5]
    \node [circle,fill=blue,inner sep=0pt,minimum size=5pt,label={[right]D}] (A) at (3,3) {};
    \node [rectangle,fill=red,inner sep=0pt,minimum size=5pt,label={[left]A}] (B) at (-3,3) {};
    \node [circle,fill=blue,inner sep=0pt,minimum size=5pt,label={[left]C}] (C) at (-3,-3) {};
    \node [rectangle,fill=red,inner sep=0pt,minimum size=5pt,label={[right]B}] (D) at (3,-3) {};
    \draw[-,thick] (A)--(C);
    \draw[-,thick] (B)--(D);
\end{tikzpicture}
	     & &
	\begin{tikzpicture}[scale=0.5]
    \node [circle,fill=blue,inner sep=0pt,minimum size=5pt,label={[right]A}] (A) at (3,0) {};
    \node [circle,fill=blue,inner sep=0pt,minimum size=5pt,label={[left]B}] (B) at (-3,0) {};
    \node [circle,fill=blue,inner sep=0pt,minimum size=5pt,label={[above]D}] (C) at (0,5) {};
    \node [rectangle,fill=red,inner sep=0pt,minimum size=5pt,label={[above]C}] (D) at (0,2) {};
    \draw[-,thick] (A)--(B)--(C)--(A);
\end{tikzpicture}
\end{tabular}
\caption{Two Radon configurations in $\mathbb{R}^2$.}
\label{fig:TwoRadonR2}
\end{center}
\end{figure}

\begin{definition}\label{def:general}
A finite set $X\subseteq\mathbb{R}^n$ is in \textit{general position} if, for any $k<n$, no $(k+2)$-subset of $X$ lies in a $k$-flat.
\end{definition}

\begin{theorem}[\cite{peterson1972geometry}]\label{thm:radon-iff}
Let $T$ be a set of $n+2$ points in $\mathbb{R}^n$.
Then $T$ is in general position if and only if the partition $\{T_1,T_2\}$ guaranteed by Radon's Theorem is unique.
\end{theorem}

This uniqueness goes even further.
Not only is the partition unique, general position also implies that the intersection of the convex hulls contains exactly one point. 

\begin{theorem}
[\cite{peterson1972geometry}]
Let $\{T_1,T_2\}$ be the Radon partition of a set of $n+2$ points in general position in $\mathbb{R}^n$.
Then $\mathrm{conv}(T_1)\cap\mathrm{conv}(T_2)$ is a single point.
\end{theorem}

In the context of support vector machines, in Section~\ref{sec:svm-general-selected} we will show that if our points are in (a slightly stronger notion of) general position, then upon projecting the positive and negative support vectors onto the optimal $(n-1)$-dimensional separating hyperplane, their convex hulls intersect in a single Radon point.

See~\cite{de2020stochastic} for a stochastic version of Radon's theorem (and the more general Tverberg theorem), along with some applications to data classification using logistic regression.

%%%%%%%%%%%%% Section %%%%%%%%%%%%%%%%%%

\begin{figure}[h]
\begin{center}
\begin{tabular}{ p{18em} c }
\begin{tikzpicture}[scale=0.6]
    \coordinate [] (A) at (-3,-3);
    \coordinate [] (B) at (-3,3);
    \coordinate [] (C) at (-1,4);
    \coordinate [] (D) at (-1,-2);
    
    \coordinate [] (H) at (1,-3);
    \coordinate [] (G) at (1,3);
    \coordinate [] (F) at (3,4);
    \coordinate [] (E) at (3,-2);

    \foreach \i in {A,B,C,D}
    \draw[-, fill=blue!30, opacity=.7] (A)--(B)--(C)--(D)--cycle;
    \draw[-, thick, fill=red!30, opacity=.5] (E)--(F)--(G)--(H)--cycle;
    
    \node [circle,fill=blue,inner sep=0pt,minimum size=3pt] (G) at (-2,0) {};
    \node [rectangle,fill=red,inner sep=0pt,minimum size=3pt] (G) at (1.5,-1) {};
    \node [rectangle,fill=red,inner sep=0pt,minimum size=3pt] (G) at (2,1.5) {};
    \node [rectangle,fill=red,inner sep=0pt,minimum size=3pt] (G) at (2.75,0.5) {};
\end{tikzpicture}
&
\begin{tikzpicture}[scale=0.6]
    \coordinate [] (A) at (-3,-3);
    \coordinate [] (B) at (-3,3);
    \coordinate [] (C) at (-1,4);
    \coordinate [] (D) at (-1,-2);
    
    \coordinate [] (H) at (1,-3);
    \coordinate [] (G) at (1,3);
    \coordinate [] (F) at (3,4);
    \coordinate [] (E) at (3,-2);

    \foreach \i in {A,B,C,D}
    \draw[-, fill=blue!30, opacity=.7] (A)--(B)--(C)--(D)--cycle;
    \draw[-, thick, fill=red!30, opacity=.5] (E)--(F)--(G)--(H)--cycle;
    
    \node [circle,fill=blue,inner sep=0pt,minimum size=3pt] (G) at (-2.5,1.5) {};
    \node [circle,fill=blue,inner sep=0pt,minimum size=3pt] (G) at (-1.5,-1) {};
    \node [rectangle,fill=red,inner sep=0pt,minimum size=3pt] (G) at (1.5,-1.5) {};
    \node [rectangle,fill=red,inner sep=0pt,minimum size=3pt] (G) at (2.75,0.5) {};
    \end{tikzpicture}\\
    
 \begin{tikzpicture}[scale=0.4]
    \node [rectangle,fill=red,inner sep=0pt,minimum size=5pt] (A) at (3,0) {};
    \node [rectangle,fill=red,inner sep=0pt,minimum size=5pt] (B) at (-3,0) {};
    \node [rectangle,fill=red,inner sep=0pt,minimum size=5pt] (C) at (0,5) {};
    \node [circle,fill=blue,inner sep=0pt,minimum size=5pt] (D) at (0,2.25) {};
    \draw[-,thick] (A)--(B)--(C)--(A);
\end{tikzpicture}
&
\begin{tikzpicture}[scale=0.4]
    \node [rectangle,fill=red,inner sep=0pt,minimum size=5pt] (A) at (3,3) {};
    \node [circle,fill=blue,inner sep=0pt,minimum size=5pt] (B) at (-3,3) {};
    \node [rectangle,fill=red,inner sep=0pt,minimum size=5pt] (C) at (-3,-3) {};
    \node [circle,fill=blue,inner sep=0pt,minimum size=5pt] (D) at (3,-3) {};
    \draw[-,thick] (A)--(C);
    \draw[-,thick] (B)--(D);
\end{tikzpicture}
\end{tabular}
\end{center}
\vspace{5mm}
\begin{center}
\begin{tabular}{ p{18em} c }
 \begin{tikzpicture}[scale=0.6]
     \coordinate [] (A) at (-3,-3);
     \coordinate [] (B) at (-3,3);
     \coordinate [] (C) at (-1,4);
     \coordinate [] (D) at (-1,-2);
    
     \coordinate [] (H) at (1,-3);
     \coordinate [] (G) at (1,3);
     \coordinate [] (F) at (3,4);
     \coordinate [] (E) at (3,-2);

     \foreach \i in {A,B,C,D}
     \draw[-, fill=blue!30, opacity=.7] (A)--(B)--(C)--(D)--cycle;
     \draw[-, thick, fill=red!30, opacity=.5] (E)--(F)--(G)--(H)--cycle;
    
     \node [circle,fill=blue,inner sep=0pt,minimum size=3pt] (G) at (-2.5,-2) {};
     \node [circle,fill=blue,inner sep=0pt,minimum size=3pt] (G) at (-1.5,2) {};
     \node [rectangle,fill=red,inner sep=0pt,minimum size=3pt] (G) at (2,0) {};

 \end{tikzpicture}
 &
 \begin{tikzpicture}[scale=0.6]
    \coordinate [] (A) at (-3,-3);
    \coordinate [] (B) at (-3,3);
    \coordinate [] (C) at (-1,4);     \coordinate [] (D) at (-1,-2);
    
    \coordinate [] (H) at (1,-3);
    \coordinate [] (G) at (1,3);
    \coordinate [] (F) at (3,4);
    \coordinate [] (E) at (3,-2);

    \foreach \i in {A,B,C,D}
        \draw[-, fill=blue!30, opacity=.7] (A)--(B)--(C)--(D)--cycle;
        \draw[-, thick, fill=red!30, opacity=.5] (E)--(F)--(G)--(H)--cycle;
    
    \node [circle,fill=blue,inner sep=0pt,minimum size=3pt] (G) at (-2,0) {};
    \node [rectangle,fill=red,inner sep=0pt,minimum size=3pt] (G) at (2,0) {};
    \end{tikzpicture}\\
     \\
  \begin{tikzpicture}[scale=0.4]
    \node [circle,fill=blue,inner sep=0pt,minimum size=5pt] (A) at (-3,0) {};
    \node [rectangle,fill=red,inner sep=0pt,minimum size=5pt] (B) at (0,0) {};
    \node [circle,fill=blue,inner sep=0pt,minimum size=5pt] (C) at (3,0) {};
    \draw[-,thick] (A)--(B)--(C);
 \end{tikzpicture}
 &
 \begin{tikzpicture}[scale=0.4]
    \node [circle,fill=purple,inner sep=0pt,minimum size=5pt] (A) at (0,0) {};
\end{tikzpicture}
\end{tabular}
\end{center}
\caption{Four possible support vector configurations in $\mathbb{R}^3$ (the figures with planes), each drawn above its corresponding projected Radon configuration in the separating hyperplane (the figures with just lines and points). 
As we will see in Section~\ref{sec:svm-general-selected}, when points are in strong general position, there can be at most $n+1$ support vectors. 
Thus in $\mathbb{R}^3$, we can have configurations with 2, 3, or 4 support vectors.
See Table~\ref{table:simulations3} in Section~\ref{sec:experiments} for relative frequencies of these configurations.
}
\label{fig:3d-sv-configurations}
\end{figure}
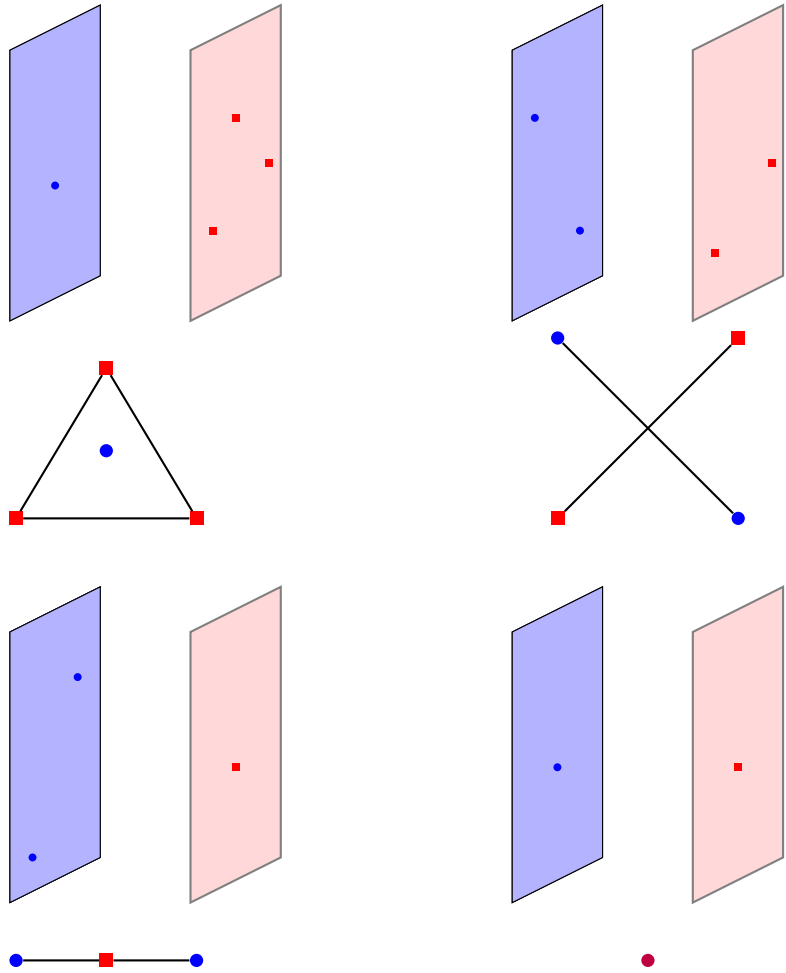

%%%%%%%%%%%%% Section %%%%%%%%%%%%%%%%%%

\section{Radon's theorem and SVMs}
\label{sec:RadonsThmandSVM}

In this paper, we show that Radon's theorem can be used as a tool to identify and classify support vector configurations.
Note that the support vectors are already partitioned into two different classes, the positive class and the negative class.
As such, Radon's theorem has consequences for SVM configurations, and for the projections of support vectors onto the separating hyperplane. 
For example, in Theorem~\ref{thm:svm-hulls-intersect} we will show that a separating hyperplane is optimal if and only if the projections of the convex hulls of the support vectors from the two classes onto this hyperplane intersect.
In $\mathbb{R}^3$, most support vector configurations will look like the figures found in Figure~\ref{fig:3d-sv-configurations}, and when we project those configurations onto the separating hyperplane we get a Radon configuration and a Radon point.
To say more about the properties of the Radon points (such as uniqueness), we will need some additional general position concepts, which will be discussed in Section~\ref{sec:svm-general-selected}.

Similar to the definition of a Radon point in Section~\ref{sec:radon}, we define a Radon point in the SVM setting.

\begin{definition}\label{def:SVM_Radon_pt}
Suppose $X_+$ and $X_-$ are linearly separable finite sets of points in $\mathbb{R}^n$, with parallel separating hyperplanes $H_+$ and $H_-$ such that $C_+=X_+\cap H_+$ and $C_-=X_-\cap H_-$ are non-empty.
Let $H$ be the parallel hyperplane midway between $H_+$ and $H_-$, and define $\rho\colon \mathbb{R}^n\to H$ to be orthogonal projection onto $H$.
A \textit{Radon point} is a point $\mathbf{v}\in H$ such that $\mathbf{v}\in \rho(\mathrm{conv}(C_+))\cap \rho(\mathrm{conv}(C_-))$.
\end{definition}

\begin{figure}[h!]
\centering
\begin{subfigure}{.5\textwidth}
  \centering
  \includegraphics[width=.6\linewidth]{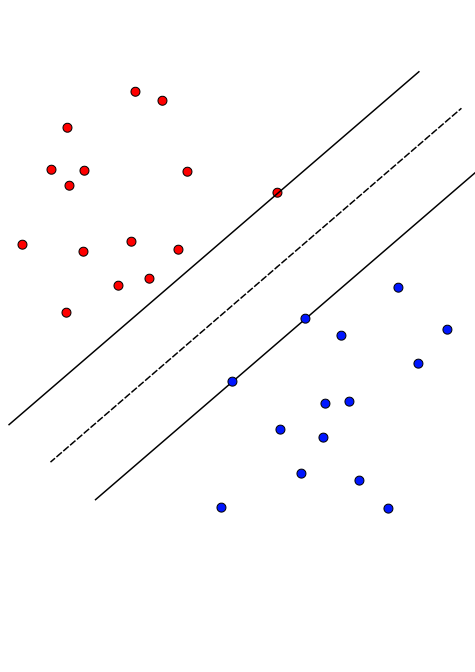}
  \caption{Suboptimal separating hyperplane}
  \label{fig:incorrectsv}
\end{subfigure}%
\begin{subfigure}{.5\textwidth}
  \centering
  \includegraphics[width=.6\linewidth]{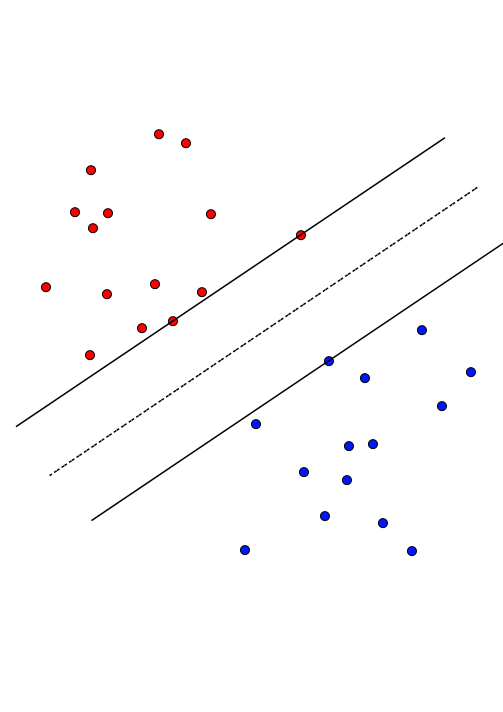}
  \caption{Optimal separating hyperplane}
  \label{fig:correctsv}
\end{subfigure}
\caption{The same dataset with two different separating hyperplanes.
Note that (A) does not contain a Radon point in the projection of the convex hulls whereas (B) does.
The margin of separation is larger in (B) than in (A).}
\label{fig:svcomparison}
\end{figure}

An example of a Radon point can be found in Figure~\ref{fig:correctsv}. 
In this case, $C_+$ is the two red points on the upper separating hyperplane while $C_-$ is the single blue point on the lower separating hyperplane. 
The Radon point is the intersection of the projection of the convex hull of two red points (a line segment) and the projection of the single blue point onto the middle separating hyperplane. 
In this example, the Radon point coincides with the projection of the blue point.
By contrast, in Figure~\ref{fig:incorrectsv} there is no Radon point.

Now, we can say something about the existence of a Radon point, and its relationship to whether a separating hyperplane is optimal.

\begin{theorem}
\label{thm:svm-hulls-intersect}
If $X\subset\mathbb{R}^n$ is a finite set of linearly separable labeled points, then the orthogonal projections of the convex hulls of the positive and negative support vectors onto the optimal separating hyperplane intersect in at least one Radon point. 
Conversely, if the intersection of the orthogonal projections of the convex hulls of the positive and negative support vectors contains a Radon point, then the separating hyperplane is optimal.
\end{theorem}

\begin{proof}
We first prove the forward direction.
By the KKT conditions in Section~\ref{ssec:solving} we have (i) $\mathbf{w}=\sum_j y_j\alpha_j\mathbf{x}_j$ and (ii) $\sum_j \alpha_j y_j=0$, where $\mathbf{w}$ is the normal vector to the optimal separating hyperplane.
Let $S_+=\{j~|~y_j=1\}$ be the set of indices for the positive class, and let $S_-=\{j~|~y_j=-1\}$ be the set of indices for the negative class.
Reorganizing (ii), we may define $c:=\sum_{j\in S_+}\alpha_j=\sum_{j\in S_-}\alpha_j$.
Since the intersection patterns of projections of the convex hulls are unchanged by translation, we may assume for the purposes of projection that $b=0$.
Let $\rho\colon\mathbb{R}^n\to\mathbb{R}^n$ be the orthogonal projection onto the orthogonal complement of $\mathbf{w}$, that is, the orthogonal projection onto the separating hyperplane.
Then (i) gives that $\mathbf{0}=\rho(\mathbf{w})=\rho(\sum_j y_j\alpha_j\mathbf{x}_j)$, which we may reorganize using the linearity of $\rho$ to get
\[\rho\left(\sum_{j\in S_+} \alpha_j\mathbf{x}_j\right) = \rho\left(\sum_{j\in S_-} \alpha_j\mathbf{x}_j\right).\]
We may rescale by $\frac{1}{c}$ to obtain the same equality for convex combinations
\[\rho\left(\sum_{j\in S_+} \frac{\alpha_j}{c}\mathbf{x}_j\right) = \rho\left(\sum_{j\in S_-} \frac{\alpha_j}{c}\mathbf{x}_j\right),\]
where we note that these combinations are convex since $1=\sum_{j\in S_+}\frac{\alpha_j}{c}=\sum_{j\in S_-}\frac{\alpha_j}{c}$.
Hence we have shown that the projections of the convex hulls of the positive and negative support vectors onto the separating hyperplane intersect.

Conversely, let $H_+$ and $H_-$ be parallel hyperplanes, where $C_+=H_+\cap X_+$ and $C_-=H_-\cap X_-$ are non-empty, and where $H_+$ and $H_-$ separate the remaining points in $X_+$ and $X_-$. 
Let $H$ be the parallel hyperplane midway between $H_+$ and $H_-$ and let $\rho$ be the orthogonal projection map onto $H$.
Suppose there exists a Radon point $\mathbf{r}\in H$ such that $\mathbf{r}\in \rho(\text{conv}(C_+))\cap \rho(\text{conv}(C_-))$.
We show $H$ is optimal, meaning that no other separating hyperplanes provide a larger margin of separation.

Let $\mathbf{r}_+\in \text{conv}(C_+)$ and $\mathbf{r}_-\in \text{conv}(C_-)$ be such that $\rho(\mathbf{r}_+)=\mathbf{r}=\rho(\mathbf{r}_-)$.
Note that the distance between $H_+$ and $H_-$ is $\|\mathbf{r}_+-\mathbf{r}_-\|$.
We want to show that for any separating hyperplanes $H'_+$ and $H'_-$, the margin between $H'_+$ and $H'_-$ is at most $\|\mathbf{r}_+-\mathbf{r}_-\|$.

Since $\mathbf{r}_+\in \text{conv}(C_+)$, we can write $\mathbf{r}_+=\sum_{i}\alpha_i \mathbf{c}_i$ where $\mathbf{c}_i\in C_+$, $\alpha_i\geq 0$, and $\sum_i \alpha_i=1$. 

Now, let $\mathbf{w}$ be any unit vector in $\mathbb{R}^n$. 
Then, 
\begin{align*}
    \min\limits_{\mathbf{x}\in X_+}(\mathbf{w}\cdot \mathbf{x})
    &=\left(\sum\limits_{i}\alpha_i\right)\min\limits_{\mathbf{x}\in X_+}(\mathbf{w}\cdot \mathbf{x}) &\text{since } \sum_i \alpha_i=1\\
    &=\sum\limits_{i}\left(\alpha_i\min\limits_{\mathbf{x}\in X_+}(\mathbf{w}\cdot \mathbf{x})\right)\\
    &\leq \sum\limits_i \alpha_i \mathbf{w} \cdot \mathbf{c}_i &\text{since } \mathbf{c}_i\in C_+\subseteq X_+\\
    &= \mathbf{w} \cdot \sum\limits_i \alpha_i \mathbf{c}_i \\
    &= \mathbf{w} \cdot \mathbf{r}_+.
\end{align*}

Similarly, we have $\max\limits_{\mathbf{x}\in X_-}(\mathbf{w}\cdot \mathbf{x})\geq \mathbf{w} \cdot \mathbf{r}_-$.

Finally, consider all unit vectors $\mathbf{w}$ such that some hyperplane normal to $\mathbf{w}$ separates $X_+$ and $X_-$, with $\mathbf{w}\cdot\mathbf{x}>0$ for all $\mathbf{x}\in X_+$.
Let $q_\mathbf{w}$ be the orthogonal projection map onto the line spanned by $\mathbf{w}$. 
Then, for any two separating hyperplanes $H'_+$ and $H'_-$ orthogonal to $\mathbf{w}$, the distance between them is
\begin{align*}
    d(H'_+,H'_-) &\leq d(q_\mathbf{w}(X_+),q_\mathbf{w}(X_-)) \\   
    &= \min\limits_{\mathbf{x}\in X_+}(\mathbf{w} \cdot \mathbf{x}) - \max\limits_{\mathbf{x}\in X_-}(\mathbf{w}\cdot \mathbf{x})\\%\text{since }\mathbf{w}\cdot\mathbf{x}>0\text{ for all }\mathbf{x}\in X_+\\
    &\leq (\mathbf{w} \cdot \mathbf{r}_+) - (\mathbf{w} \cdot \mathbf{r}_-)\\% & \text{by above}\\
    &= \mathbf{w}\cdot(\mathbf{r}_+ - \mathbf{r}_-)\\
    &\leq \|\mathbf{r}_+ - \mathbf{r}_-\| & \text{since }\|\mathbf{w}\|=1.
\end{align*}
Thus, the distance between any two separating hyperplanes is at most $\|\mathbf{r}_+ - \mathbf{r}_-\|$, which is achieved when $\mathbf{w}$ points in the direction of $\mathbf{r}_+ - \mathbf{r}_-$.
This implies that $H_+$ and $H_-$ are optimal.
\end{proof}

Thus if we have a ``claimed'' support vector configuration where the projection of the convex hulls does not result in at least one Radon point, then we know the choice of SVM separating hyperplane was incorrect (e.g.\ as might arise in the case of inaccuracies as a result of rounding errors).
Consider Figure~\ref{fig:svcomparison}. 
The two datasets in (A) and (B) are the same, but (A) has support vectors where if one projects their convex hulls, they do not intersect. 
By Theorem~\ref{thm:svm-hulls-intersect}, these must not be the correct support vectors.
Indeed, (B) does have a Radon point in the projection of the convex hulls of the support vectors, and the margin between the two datasets is wider. 
By Theorem~\ref{thm:svm-hulls-intersect}, (B) has the correct optimal separating hyperplane.

To get stronger results, we will need a stronger notion of general position, which we introduce in Section~\ref{sec:svm-general-selected}.
These results include showing that the projected convex hulls of the optimal support vectors intersect in a single Radon point and no more, and that the identities of the support vectors are stable under small perturbations of the data points. 

%\note{Note that Theorem~\ref{thm:svm-hulls-intersect} implies that Figure~1 of  (\footnote{\url{https://www.microsoft.com/en-us/research/wp-content/uploads/2016/02/tr-98-14.pdf}}) is not yet optimized since the projections of convex hulls of the support vectors don't intersect.
%Also, see slide 7 of (\footnote{\url{http://web.mit.edu/6.034/wwwbob/svm-notes-long-08.pdf}}) for an incorrect drawing where the support vectors contradict Theorem~\ref{thm:svm-hulls-intersect}!}

%%%%%%%%%%%%% Section %%%%%%%%%%%%%%

\section{SVMs for points in general position}
\label{sec:svm-general-selected}

There are a wide variety of notions of general position; in Definition~\ref{def:general} we gave only one such notion.
We begin in Section~\ref{ssec:SGP-def} by defining a stronger notion of general position that will be useful in the context of support vector machines.
In Section~\ref{ssec:SGP-generic} we show that strong general position is a generic property.
Finally, in Section~\ref{ssec:SGP-SVM} we show that for points in strong general position, a sufficiently small perturbation cannot change which data points are labeled as support vectors, and we give the basic properties of SVMs that hold in this slightly more restrictive setting of strong general position.

\subsection{The definition of strong general position}
\label{ssec:SGP-def}

\begin{definition}\label{def:sgp2}
A finite set of points $X\subseteq\mathbb{R}^n$ is in \textit{strong general position} if all three of the following conditions are met.
\begin{enumerate}
    \item[$(i)$] For integer $0\leq k \leq n-1$, no $k$-flat contains more than $k+1$ points of $X$.
    \item[$(ii)$] No disjoint $k$-flats and $\ell$-flats constructed as the affine spans of sets of points in $X$ contain parallel vectors.
    \item[$(iii)$] Let $\mathbf{u}_0,\ldots,\mathbf{u}_{k}\subseteq X$ generate a $k$-flat $U$, and let $\mathbf{v}_0,\ldots,\mathbf{v}_{\ell}\subseteq X$ generate a disjoint $\ell$-flat $V$.
    Let $\mathbf{w}$ be the vector with head in $U$ and tail in $V$ whose length is equal to the distance between $U$ and $V$.
    We require that the hyperplane normal to $\mathbf{w}$ that contains $U$ contains no points in $X$ other than $\mathbf{u}_0,\ldots,\mathbf{u}_{k}$.
\end{enumerate}
\end{definition}

We note that a necessary and sufficient condition for $\mathbf{w}$ to minimize the distance between $U$ and $V$ is that the vector $\mathbf{w}$ is orthogonal to both $U$ and $V$.
This explains why there indeed is a hyperplane normal to $\mathbf{w}$ which contains $U$.
By symmetry, $(iii)$ also allows us to conclude that the hyperplane normal to $\mathbf{w}$ that contains $V$ furthermore contains no points in $X$ other than $\mathbf{v}_0,\ldots,\mathbf{v}_{\ell}$.
The fact that the vector $\mathbf{w}$ in $(iii)$ is unique follows from $(ii)$, which implies that $U$ is not parallel to a subspace of $V$, and similarly $V$ is not parallel to a subspace of $U$.
See Figure~\ref{fig:degenerate_cases} for two examples of points that are not in strong general position.

\begin{figure}[h]
\begin{tabular}{p{18em} c}
\begin{tikzpicture}[scale=0.5]
    \node [circle,fill=blue,inner sep=0pt,minimum size=5pt, label={$\mathbf{p}_0$}] (B) at (-3,3) {};
    \node [circle,fill=blue,inner sep=0pt,minimum size=5pt, label={$\mathbf{p}_1$}] (C) at (-3,-3) {};
    \node [rectangle,fill=red,inner sep=0pt,minimum size=5pt, label={$\mathbf{n}_1$}] (D) at (3,-2) {};
    \node [rectangle,fill=red,inner sep=0pt,minimum size=5pt, label={$\mathbf{n}_0$}] (E) at (3,1) {};
    \draw[-,thick] (0,4)--(0,-4);

\end{tikzpicture}
&
 \begin{tikzpicture}[scale=0.6]
    \coordinate [] (A) at (-3,-3);
    \coordinate [] (B) at (-3,3);
    \coordinate [] (C) at (-1,4);     
    \coordinate [] (D) at (-1,-2);
    
    \coordinate [] (H) at (1,-3);
    \coordinate [] (G) at (1,3);
    \coordinate [] (F) at (3,4);
    \coordinate [] (E) at (3,-2);

    \foreach \i in {A,B,C,D}
        \draw[-, fill=blue!30, opacity=.7] (A)--(B)--(C)--(D)--cycle;
        \draw[-, thick, fill=red!30, opacity=.5] (E)--(F)--(G)--(H)--cycle;
    
    \node [circle,fill=blue,inner sep=0pt,minimum size=3pt, label={$\mathbf{p}_0$}] (G) at (-2,0) {};
    \node [rectangle,fill=red,inner sep=0pt,minimum size=3pt, label={$\mathbf{n}_0$}] (G) at (2,2) {};
    \node [rectangle,fill=red,inner sep=0pt,minimum size=3pt, label={$\mathbf{n}_1$}] (G) at (2,-1) {};
    \node [rectangle,fill=red,inner sep=0pt,minimum size=3pt, label={$\mathbf{n}_2$}] (G) at (2.5,0) {};
    \end{tikzpicture}
\end{tabular}
    \caption{Two examples of point sets $X\subseteq \mathbb{R}^n$ that are not in strong general position.
    (The point sets are unlabeled, and the colors are for illustration.)
    On the left, in $\mathbb{R}^2$, the vector $\protect\overrightarrow{\mathbf{p}_0\mathbf{p}_1}$ is parallel to $\protect\overrightarrow{\mathbf{n}_0\mathbf{n}_1}$, violating condition $(ii)$ of strong general position.
    On the right, we provide an example in $\mathbb{R}^3$.
    If $\mathbf{w}$ is the shortest vector between $\mathbf{p}_0$ and the line between $\mathbf{n}_0$ and $\mathbf{n}_1$, then the hyperplane perpendicular to $\mathbf{w}$ that contains $\mathbf{n}_0,\mathbf{n}_1$ also contains $\mathbf{n}_2$, violating condition $(iii)$.
    In other words, the orthogonal projection of $\mathbf{p}_0$ onto the affine span of $\mathbf{n}_0,\mathbf{n}_1,\mathbf{n}_2$ lies on the line segment between $\mathbf{n}_0$ and $\mathbf{n}_1$, which differs from the top left in Figure~\ref{fig:3d-sv-configurations} where instead the projection of $\mathbf{p}_0$ is in the interior of the convex hull of $\mathbf{n}_0$, $\mathbf{n}_1$, and $\mathbf{n}_2$.}
    \label{fig:degenerate_cases}
\end{figure}
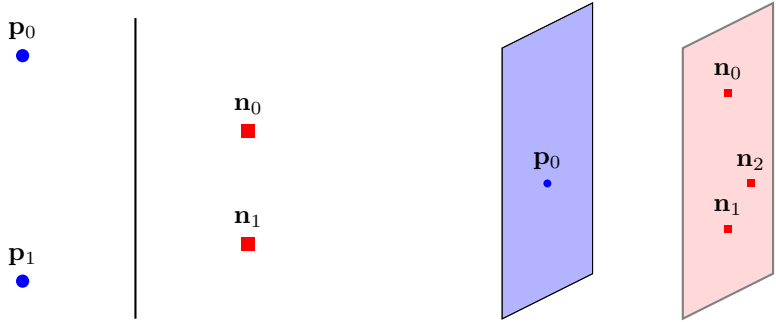

%%% Subsection %%%

\subsection{Strong general position is a generic property}
\label{ssec:SGP-generic}

We want to show that strong general position is a generic property, meaning if we take a random set of linearly separable points, they satisfy the conditions of strong general position with probability 1. 
To show that these configurations are generic, we look to the Zariski topology. 
The Zariski topology is defined on affine space over a field $\mathbb{F}$ where the closed sets are the algebraic subsets in the affine space~\cite{dummit2004abstract}.
For our purposes, we let $\mathbb{F}=\mathbb{R}$ and look at the Zariski topology on $\mathbb{R}^n$. 
The Zariski-closed sets are the affine varieties of $\mathbb{R}^n$, i.e.\ the vanishing set $V(A)$ of all of the polynomials in some subset $A\subset\mathbb{R}[x_1,\ldots,x_n]$. 
This means we obtain the Zariski-open sets by taking the complements $\mathbb{R}^n-V(A)$ of algebraic varieties.
Note, the non-empty Zariski-open sets in $\mathbb{R}^n$ are open and dense in the regular Euclidean topology (see~\cite[Proposition~1, Page 148]{blum1998complexity}, for example).
% https://link.springer.com/content/pdf/bbm%3A978-1-4612-0701-6%2F1%2F1
% https://link.springer.com/book/10.1007/978-1-4612-0701-6
%\url{https://math.ucr.edu/~res/math205A-2014/zariski-topology.pdf}
%\note{Theorem 3.1.6 at \url{https://www.math.tamu.edu/~frank.sottile/teaching/19.1/Chapters/Ch3.pdf}.}
Thus to show strong general position is a generic property, we will show that the configurations of $m$ points in $\mathbb{R}^n$ that do not satisfy strong general position live in a Zariski-closed subset of $\mathbb{R}^{mn}$.

While the above conditions $(i)$--$(iii)$ for strong general position in Definition~\ref{def:sgp2} are intuitive, for ease of showing genericity, we introduce two alternate axioms for $(i)$ and $(ii)$, for a finite set of points $X\subseteq\mathbb{R}^n$.
\begin{enumerate}
    \item[$(i')$] For every set of distinct points $\mathbf{u}_0, \mathbf{u}_1, \ldots, \mathbf{u}_k\in X$ with $k\leq n$, the vectors $\mathbf{u}_1-\mathbf{u}_0, \ldots, \mathbf{u}_k-\mathbf{u}_0$ are linearly independent.
    \item[$(ii')$] For every set of distinct points $\mathbf{u}_0, \ldots, \mathbf{u}_k, \mathbf{v}_0, \ldots, \mathbf{v}_\ell \in X$ with $k+\ell\leq n$, the vectors $\mathbf{u}_1-\mathbf{u}_0, \ldots, \mathbf{u}_k-\mathbf{u}_0,\mathbf{v}_1-\mathbf{v}_0, \ldots, \mathbf{v}_\ell-\mathbf{v}_0$ are linearly independent.
\end{enumerate}

Note that $(i')$ can be seen as a special case of $(ii')$ when $\ell=0$.

We first show that these axioms are equivalent to their respective counterparts.

\begin{theorem}\label{thm:axiom_equiv}
Condition $(i')$ is equivalent to condition $(i)$. 
Condition $(ii')$ implies condition $(ii)$. 
Finally, conditions $(i)$ and $(ii)$ imply $(ii')$.
\end{theorem}

\begin{proof}
\textit{Condition $(i')$ is equivalent to condition $(i)$.}
By the correspondence between linear and affine independence, condition $(i')$ is equivalent to the statement that for $k\le n$, any set of $k$ points in $X$ is affinely independent.
This statement is also equivalent to $(i)$, which is easiest to see if one reindexes in order to rephrase $(i)$ as: ``For integer $1\le k\le n$, no $(k-1)$-flat contains more than $k$ points of $X$.''

\textit{Condition $(ii')$ implies condition $(ii)$.}
Suppose condition $(ii')$ holds.
Furthermore, suppose $U$ is a $k$-flat in $\mathbb{R}^n$ generated as the affine span of $\mathbf{u}_0, \ldots, \mathbf{u}_k$, suppose $V$ is an $\ell$-flat in $\mathbb{R}^n$ affinely generated by  $\mathbf{v}_0,\ldots, \mathbf{v}_\ell$, and suppose $U$ and $V$ are disjoint.
In order to show that $(ii)$ holds, we need to show that $U$ and $V$ do not contain parallel vectors.

We first argue that $k+\ell<n.$ 
Suppose not: suppose $k+\ell\geq n$.
Remove arbitrary points from $U$ and $V$ such that $k+\ell=n$. 
Choose some vector $\mathbf{w}$ such that $\mathbf{w}$ is a shortest vector between $U$ and $V$. 
By applying Lemma~\ref{lem:min_Euclidean_norm} to the affine space $U-V=\{\mathbf{u}-\mathbf{v}~|~\mathbf{u}\in U,\ \mathbf{v}\in V\}$, we are guaranteed that such a $\mathbf{w}$ exists.
Further, $\mathbf{w}$ is perpendicular to each of the $k+\ell=n$ vectors.
But, $(ii')$ states that these $n$ vectors are linearly independent.
This implies that $\mathbf{w}=\mathbf{0}$ and thus that $U$ and $V$ are not disjoint, a contradiction.
Therefore, $k+\ell<n$. 

Furthermore, the fact that $U$ and $V$ are $k$- and $\ell$-flats generated by $k+1$ and $\ell+1$ points, respectively, implies that the points are distinct.
As a consequence, we have satisfied the hypotheses for $(ii'),$ and we can conclude that $\mathbf{u}_1-\mathbf{u}_0, \ldots, \mathbf{u}_k-\mathbf{u}_0,\mathbf{v}_1-\mathbf{v}_0, \ldots, \mathbf{v}_\ell-\mathbf{v}_0$ are linearly independent.

Now suppose $U$ and $V$ do contain parallel non-zero vectors: that is, suppose $\mathbf{u}=\lambda\mathbf{v},$ where $\mathbf{u}$ is a non-zero displacement vector defined by distinct points in $U$, and $\mathbf{v}$ is a non-zero displacement vector defined by distinct points in $V$.
Then, by definition of $U$ and $V,$ we can write $\mathbf{u}=\sum_{i=0}^k\alpha_i(\mathbf{u}_i-\mathbf{u}_0)$ with $\sum_{i=0}^k\alpha_i=1$, and we can write  $\mathbf{v}=\sum_{j=0}^{\ell}\beta_j(\mathbf{v}_j-\mathbf{v}_0)$ with $\sum_{j=0}^{\ell}\beta_j=1.$ But then $\mathbf{0}=\sum_{i=0}^k\alpha_i(\mathbf{u}_i-\mathbf{u}_0)-\sum_{j=0}^{\ell}\lambda\beta_j(\mathbf{v}_j-\mathbf{v}_0)$, and certainly not all coefficients are zero.
This contradicts the the linear independence of $\mathbf{u}_1-\mathbf{u}_0,\ldots,\mathbf{u}_k-\mathbf{u}_0,\mathbf{v}_1-\mathbf{v}_0,\ldots,\mathbf{v}_{\ell}-\mathbf{v}_0$ guaranteed by $(ii')$.
Hence, $(ii')$ implies $(ii)$.

\textit{Conditions $(i)$ and $(ii)$ imply $(ii')$.}
Assume that conditions $(i)$ and $(ii)$ hold. 
As in condition $(ii')$, consider distinct points $\mathbf{u}_0, \ldots, \mathbf{u}_k$ and $\mathbf{v}_0, \ldots, \mathbf{v}_\ell$ such that $k+\ell\leq n$.
Again, let $U$ and $V$ be the $k$-flat and $\ell$-flat defined as the affine spans of the points $\mathbf{u}_0, \ldots, \mathbf{u}_k$ and $\mathbf{v}_0,\ldots, \mathbf{v}_\ell$, respectively.
We show that $\mathbf{u}_1-\mathbf{u}_0, \ldots, \mathbf{u}_k-\mathbf{u}_0,\mathbf{v}_1-\mathbf{v}_0, \ldots, \mathbf{v}_\ell-\mathbf{v}_0$ are linearly independent.

Since $(i)$ implies $(i')$, we have that the vectors $\mathbf{u}_1-\mathbf{u}_0, \ldots, \mathbf{u}_k-\mathbf{u}_0$ are linearly independent.
Similarly, the vectors $\mathbf{v}_1-\mathbf{v}_0, \ldots, \mathbf{v}_\ell-\mathbf{v}_0$ are linearly independent.

In search of contradiction, suppose there exist coefficients $c_0,\ldots, c_k$ and $c'_0,\ldots, c'_\ell$ such that $\mathbf{0}=\sum\limits_{i=1}^{k}c_i(\mathbf{u}_i-\mathbf{u}_0)+\sum\limits_{j=1}^{\ell}c'_j(\mathbf{v}_j-\mathbf{v}_0)$,
where at least one $c_i$ and at least one $c'_j$ are non-zero (note that we can choose one of each to be non-zero, since if not, it would contradict the linear independence of the remaining set of vectors).
But then $U$ and $V$ contain parallel vectors, a contradiction to $(ii)$.
Therefore, $\mathbf{u}_1-\mathbf{u}_0, \ldots, \mathbf{u}_k-\mathbf{u}_0,\mathbf{v}_1-\mathbf{v}_0, \ldots, \mathbf{v}_\ell-\mathbf{v}_0$ are linearly independent, and we have shown that $(i)$ and $(ii)$ imply $(ii')$.
\end{proof}

Using these new conditions, we can show that all three conditions are satisfied on a non-empty Zariski subset of $\mathbb{R}^{mn}$.
This will allow us to conclude that the support vectors are stable.

\begin{theorem}
\label{thm:zariski}
For $X\subseteq \mathbb{R}^n$ with $|X|=m<\infty$,
the conditions $(i)$, $(ii)$, and $(iii)$ in the definition of strong general position are satisfied on a non-empty Zariski-open subset of $\mathbb{R}^{mn}$.
\end{theorem}

\begin{proof}
By Theorem~\ref{thm:axiom_equiv}, we instead consider conditions $(i')$, $(ii')$, and $(iii)$.

First, we show that $(i')$ and $(ii')$ are satisfied on a Zariski-open subset of $\mathbb{R}^{mn}$.

Consider a set $X\subseteq \mathbb{R}^n$ with $|X|=m.$ 
Condition $(i')$ fails precisely when there exists a subset of $k+1$ distinct points $\mathbf{u}_0, \mathbf{u}_1, \ldots, \mathbf{u}_k\in X$ with $k\leq n$ where the vectors $\mathbf{u}_1-\mathbf{u}_0, \ldots, \mathbf{u}_k-\mathbf{u}_0$ are linearly dependent. 
By Lemma~\ref{lem:lin_dep}, the vectors $\mathbf{u}_1-\mathbf{u}_0, \ldots, \mathbf{u}_k-\mathbf{u}_0$ are linearly dependent when $\det(A^TA)=0$, where the columns of $A$ are given by $\mathbf{u}_1-\mathbf{u}_0, \ldots, \mathbf{u}_k-\mathbf{u}_0$.
Since the determinant is polynomial in its entries, the set of points where condition $(i')$ fails in this way form an algebraic variety. 
And since finite unions of varieties are varieties, taking the union over the varieties corresponding to each choice of $k+1$ distinct points from $X$ yields the variety that contains all points for which condition $(i')$ fails. 
Therefore, we conclude that $(i')$ holds on the complement of a Zariski-closed subset of $\mathbb{R}^{mn}$, that is, a Zariski-open subset of $\mathbb{R}^{mn}$. 

The argument that $(ii')$ is satisfied on a Zariski-open subset of $\mathbb{R}^{mn}$ is similar.

Next, we restrict to the Zariski-open set where $(i')$ and $(ii')$ hold, and we show that condition $(iii)$ also holds true on a non-empty Zariski-open subset.
Let $\mathbf{u}_0,\ldots,\mathbf{u}_k\subseteq X$ generate a $k$-flat $U$, and let $\mathbf{v}_0,\ldots,\mathbf{v}_{\ell}\subseteq X$ generate a disjoint $\ell$-flat $V$.
We can find a shortest vector $\mathbf{w}$ from $V$ to $U$ by applying Lemma~\ref{lem:min_Euclidean_norm} to $U-V$,
whose elements can be written as:
\begin{align}
\label{eqn:w_vec}
\left[\mathbf{u}_0+\sum\limits_{i=1}^{k}s_i(\mathbf{u}_i-\mathbf{u}_0)\right]&-\left[\mathbf{v}_0+\sum\limits_{j=1}^\ell t_j(\mathbf{v}_j-\mathbf{v}_0)\right]\\
&=(\mathbf{u}_0-\mathbf{v}_0)+\sum\limits_{i=1}^k s_i(\mathbf{u}_i-\mathbf{u}_0)-\sum\limits_{j=1}^\ell t_j(\mathbf{v}_j-\mathbf{v}_0). \nonumber
\end{align}
Recall that $(iii)$ holds when the hyperplane normal to $\mathbf{w}$ that contains $U$ contains no points in $X$ other than $\mathbf{u}_0,\ldots,\mathbf{u}_k$.
Equation~\eqref{eqn:w_vec} implies that the entries of $\mathbf{w}$ are rational functions of the coordinates of the vectors $\mathbf{u}_0-\mathbf{v}_0$, $\mathbf{u}_i-\mathbf{u}_0$, and $\mathbf{v}_j-\mathbf{v}_0$.
Further, the denominator of the rational functions may be taken to be $\det(A^T A)$, where $A=[\mathbf{u}_1-\mathbf{u}_0,\ldots,\mathbf{u}_k-\mathbf{u}_0,\mathbf{v}_1-\mathbf{v}_0,\ldots,\mathbf{v}_\ell-\mathbf{v}_0]$.
This determinant is non-zero by Lemma~\ref{lem:lin_dep}, since the columns of $A$ are linearly independent by condition $(ii')$.
Let $\mathbf{w}'$ denote the vector obtained by multiplying $\mathbf{w}$ by the determinant $\det(A^T A)$.
Since we are restricting to the Zariski-open set on which $(i')$ and $(ii')$ hold, the equation
\[\mathbf{w}'\cdot(\mathbf{x}-\mathbf{u}_0)=0\]
for $\mathbf{x}\in X-\{\mathbf{u}_0,\ldots, \mathbf{u}_k,\mathbf{v}_0,\ldots,\mathbf{v}_\ell\}$ gives a system of polynomial equations that vanish exactly at the points where condition $(iii)$ fails.
Therefore, the set of points that satisfies conditions $(i')$, $(ii')$, and $(iii)$ is open in the Zariski topology.

To see that the set satisfying $(i')$, $(ii')$, and $(iii)$ is non-empty, note that the overarching Zariski-open set is the finite intersection of Zariski-open sets that satisfy the three conditions. We will show that each set is open, dense, and non-empty, and thus their intersection is as well.

As noted in the first half of the proof, condition $(i')$ fails when $X$ admits a subset of $k+1$ distinct points $\mathbf{u}_0, \mathbf{u}_1, \ldots, \mathbf{u}_k$ with $k\leq n$ where the vectors $\mathbf{u}_1-\mathbf{u}_0, \ldots, \mathbf{u}_k-\mathbf{u}_0$ are linearly dependent. And this happens when $\det(A^TA)=0$, where the columns of $A$ are given by $\mathbf{u}_1-\mathbf{u}_0, \ldots, \mathbf{u}_k-\mathbf{u}_0$. 
Fix a subset of $k+1$ points in $X$ and note that the complement of the variety defined by $\det(A^TA)=0$ is non-empty since it has codimension one. Thus the complement of this variety is open and dense and non-empty in $\mathbb{R}^n.$ This is true for each subset of $k+1$ points in $X.$ Thus the Zariski open set on which $(i')$ is satisfied (the finite intersection of these open and dense and non-empty sets) is also non-empty.

The argument that the Zariski-open subsets on which $(ii')$ and $(iii)$ are satisfied are non-empty is similar; both follow from the polynomial equations that define the associated varieties as in the first half of the proof.

Consequently, the intersection of these three open, dense, non-empty Zariski open sets is non-empty.

\end{proof}

As outlined above, if a subset of $\mathbb{R}^n$ is Zariski-open, that implies that the subset is open and dense in the metric topology on $\mathbb{R}^n$.
Thus, this shows that for a set $X\subseteq\mathbb{R}^n$ with $|X|=m<\infty$, being in strong general position is a generic property: the collection of configurations in strong general position forms an open and dense subset of all possible configurations, using the standard topology of $\mathbb{R}^{mn}$.

%%% subsection %%%

\subsection{Strong general position and SVMs}
\label{ssec:SGP-SVM}

Now that we have a stronger notion of general position, we can say more about the possible configurations of support vectors.
We show that for linearly separable points $X \subseteq \mathbb{R}^n$ in strong general position, there is a unique Radon point precisely when the separating hyperplane is optimal.

\begin{lemma}\label{lem:svm-hulls-intersect-pt}
If $X\subset\mathbb{R}^n$ is a set of linearly separable labeled points in strong general position, then the projections of the convex hulls of the positive and negative support vectors onto a separating hyperplane intersect at a single Radon point if and only if the separating hyperplane is optimal.
\end{lemma}

\begin{proof}
Suppose $X$ is a set of linearly separable labeled points in strong general position.
For the forward direction, note that by Theorem~\ref{thm:svm-hulls-intersect}, if the intersection of the projections of the convex hulls of the positive and negative support vectors contains at least one Radon point, then the separating hyperplane $H$ is optimal.
Thus, if the intersection is a single Radon point, it follows that $H$ is optimal. 

Conversely, let $H$ be the optimal separating hyperplane.
By Theorem~\ref{thm:svm-hulls-intersect}, the projections of the convex hulls of the positive and negative support vectors onto $H$ intersect in at least one point.
We will show that the intersection contains only one point when $X$ is in strong general position.
For $\mathbb{R}^1$, the intersection is a single point since the projections of two support vectors intersect in a point.
Now for $\mathbb{R}^n$ with $n>1$, in search of contradiction, suppose the intersection of the projections of the convex hulls of the positive and negative support vectors contains at least two distinct points, $\mathbf{x}\neq \mathbf{x}'$.
Note that the intersection of the two projections of convex hulls is convex, and hence the intersection contains the entire line segment between $\mathbf{x}$ and $\mathbf{x}'$.
It follows that the affine span of the positive support vectors contains a 1-flat parallel to a 1-flat in the affine span of the negative support vectors.
This is a contradiction since $X$ is in strong general position; see Definition~\ref{def:sgp2}$(ii)$.
Hence there is a unique Radon point.
\end{proof}

In Theorem~\ref{thm:perturbation} we will prove that for linearly separable points in strong general position, the set of support vectors is robust to perturbation.
The following result, Theorem~\ref{thm:small_perturb}, has only slightly milder hypotheses, and will imply Theorem~\ref{thm:perturbation} as an immediate consequence.
We give Theorem~\ref{thm:small_perturb} as a separate result so that we can also use it in the proof of Theorem~\ref{thm:prescribed-number-support-vectors}.

\begin{theorem}
\label{thm:small_perturb}
Let $X=X_+\cup X_-$ be a set of linearly separable labeled points in $\mathbb{R}^n$, with optimal parallel separating hyperplanes $H_+$ and $H_-$.
Let $C_+:=H_+\cap X_+$ and let $C_-:=H_-\cap X_-$.
Suppose the following two conditions hold:
\begin{itemize}
    \item[(a)] If $C_+=\{\mathbf{u}_0, \ldots, \mathbf{u}_k\}$ and $C_-=\{\mathbf{v}_0, \ldots, \mathbf{v}_\ell\}$, then the vectors $\mathbf{u}_1-\mathbf{u}_0, \ldots, \mathbf{u}_k-\mathbf{u}_0,\mathbf{v}_1-\mathbf{v}_0, \ldots, \mathbf{v}_\ell-\mathbf{v}_0$ are linearly independent.
    \item[(b)] There is a unique Radon point $\mathbf{r}$, and the points $\mathbf{r}_+\in H_+$ and $\mathbf{r}_-\in H_-$ (with $\rho(\mathbf{r}_+)=\mathbf{r}=\rho(\mathbf{r}_-)$) live in the interiors of $\mathrm{conv}(C_+)$ and $\mathrm{conv}(C_-)$, respectively.
\end{itemize}
Then, any sufficiently small perturbation of the points will preserve the identities of the support vectors.
\end{theorem}

\begin{proof}

We break the argument that follows into two steps.
First, we consider the easier case in which prior to perturbations, all of the data points are support vectors.
In this case, we show that any sufficiently small perturbation preserves the identities of the support vectors and moves the normal vector to the optimal separating hyperplane by an arbitrarily small amount.
Afterward, we handle the general case in which not all of the data points are support vectors (which is the typical setting).

We first consider the case in which prior to perturbations, all of the data points are support vectors, i.e.\ we assume $X=C_+\cup C_-$.
Since condition $(a)$ holds, we can apply Lemma~\ref{lem:min_Euclidean_norm} to the linearly independent vectors $\mathbf{u}_1-\mathbf{u}_0, \ldots, \mathbf{u}_k-\mathbf{u}_0,\mathbf{v}_1-\mathbf{v}_0, \ldots, \mathbf{v}_\ell-\mathbf{v}_0$ to see that there is a unique choice of coefficients $s_i$ and $t_j$ minimizing the Euclidean norm of the vector
\begin{equation}
\label{eqn:w_vec_2}
\mathbf{w}=(\mathbf{u}_0 -\mathbf{v}_0)+\sum\limits_{i=1}^k s_i(\mathbf{u}_i -\mathbf{u}_0) -\sum\limits^\ell_{j=1}t_j(\mathbf{v}_j -\mathbf{v}_0).
\end{equation}
Furthermore, Lemma~\ref{lem:min_Euclidean_norm} implies that the values $s_i$, $t_j$, and the entries of $\mathbf{w}$ are rational functions of the coordinates of the vectors $\mathbf{u}_0-\mathbf{v}_0$, $\left\{\mathbf{u}_i-\mathbf{u}_0\right\}_{i=1}^k$, and $\left\{\mathbf{v}_j-\mathbf{v}_0\right\}_{j=1}^{\ell}$, whose denominators may be taken to be $\det(A^T A)$, where $A=[\mathbf{u}_1-\mathbf{u}_0,\ldots,\mathbf{u}_k-\mathbf{u}_0,\mathbf{v}_1-\mathbf{v}_0,\ldots,\mathbf{v}_\ell-\mathbf{v}_0]$.
This determinant is non-zero by Lemma~\ref{lem:lin_dep}.
Define $s_0=1-\sum_{i=1}^k s_i$ and $t_0=1-\sum_{j=1}^{\ell} t_j$, and observe that $s_0$ and $t_0$ are also rational functions of the coordinates of the vectors $\mathbf{u}_0-\mathbf{v}_0$, $\left\{\mathbf{u}_i-\mathbf{u}_0\right\}_{i=1}^k$, and $\left\{\mathbf{v}_j-\mathbf{v}_0\right\}_{j=1}^{\ell}$, with the same expression in the denominator.

Since $\mathbf{r}_+$ is in the interior of $\mathrm{conv}(C_+)$ by (b), we can write $\mathbf{r}_+=\sum_{i=0}^k s'_i\mathbf{u}_i$ with $\sum_{i=0}^k s'_i=1$ and with $s'_i$ positive for all $i$.
Similarly, we can write $\mathbf{r}_-=\sum_{j=0}^{\ell} t'_j \mathbf{v}_j$ with $\sum_{j=0}^\ell t'_j=1$ and with $t'_j$ positive for all $j$.
Since $\mathbf{w}$ is the minimum-length vector from $H_-$ to $H_+$, it is normal to the optimal separating hyperplane $H$.
Since $\rho$ is the orthogonal projection onto $H$, since $\rho(\mathbf{r}_+)=\rho(\mathbf{r}_-)$, and since $\mathbf{r}_+\in H_+$ and $\mathbf{r}_-\in H_-$, it follows that $\mathbf{w}=\mathbf{r}_+ - \mathbf{r}_-$.
We rearrange to obtain
\begin{align*}
\mathbf{w}
&=\mathbf{r}_+ - \mathbf{r}_- \\
&=\sum\limits_{i=0}^k s'_i\mathbf{u}_i-\sum\limits^\ell_{j=0}t'_j\mathbf{v}_j \\
&=s'_0\mathbf{u}_0+\sum\limits_{i=1}^k s'_i\mathbf{u}_i-t'_0\mathbf{v}_0 -\sum\limits^\ell_{j=1}t'_j\mathbf{v}_j \\
&=\left(1-\sum\limits_{i=1}^k s'_i\right)\mathbf{u}_0+\sum\limits_{i=1}^k s'_i\mathbf{u}_i-\left(1-\sum\limits^\ell_{j=1}t'_j\right)\mathbf{v}_0 -\sum\limits^\ell_{j=1}t'_j\mathbf{v}_j \\
&=(\mathbf{u}_0 -\mathbf{v}_0)+\sum\limits_{i=1}^k s'_i(\mathbf{u}_i -\mathbf{u}_0) -\sum\limits^\ell_{j=1}t'_j(\mathbf{v}_j -\mathbf{v}_0).
\end{align*}
Since the coefficients in Equation~\eqref{eqn:w_vec_2} are unique, it follows that $s_i=s'_i$ and $t_j=t'_j$ for all $1\le i\le k$ and $1\le j\le \ell$.
Furthermore, $s_0=1-\sum_{i=1}^k s_i=1-\sum_{i=1}^k s'_i=s'_0$, and similarly $t_0=t'_0$.
It follows that $s_i=s'_i$ and $t_j=t'_j$ for all $0\le i\le k$ and $0\le j\le \ell$.
Therefore, $s_i$ and $t_j$ are positive for all $0\le i\le k$ and $0\le j\le \ell$.

Since nonzero determinants remain nonzero upon sufficiently small perturbation, a sufficiently small perturbation of the points in $C_+\cup C_-$ preserves that property $(a)$ still holds for the perturbed points (by Lemma~\ref{lem:lin_dep}).
So, the determinants that appear in the rational expressions (from Lemma~\ref{lem:min_Euclidean_norm}) for $\mathbf{w}$, $s_i$, and $t_j$ remain nonzero under sufficiently small perturbations of the data.
Therefore, we have well-defined
rational equations for $\mathbf{w}$, $s_i$, $t_j$, in this range of perturbations, in terms of the variables which are the entries of the vectors $\mathbf{u}_i$ and $\mathbf{v}_j$.
Since rational expressions are continuous in their domains of definition, it follows that in a sufficiently small range of perturbations, the numbers $s_i$ and $t_j$ for $0\le i\le k$ and $0\le j\le \ell$ change by a bounded amount, i.e.\ we can assume they remain positive.
So the perturbed versions of the points $\mathbf{r}_+=\sum_{i=0}^k s_i \mathbf{u}_i$ and $\mathbf{r}_-=\sum_{j=0}^{\ell} t_j\mathbf{v}_j$ remain in the interiors of the perturbed versions of the convex hulls $\mathrm{conv}(C_+)$ and $\mathrm{conv}(C_-)$, and so in this range of perturbations there still exists a Radon point.
It thus follows from Theorem~\ref{thm:svm-hulls-intersect} that the perturbed hyperplane $H$ normal to the perturbed vector $\mathbf{w}$ remains optimal under sufficiently small perturbations.
Lastly, by the continuity of the rational equations in this range, the endpoints of the vector $\mathbf{w}$ and hence also the direction of $\mathbf{w}$ move by a bounded amount, under sufficiently small perturbations.

Next, we consider the general case in which not all of the data points need to be support vectors --- i.e., in which we have no assumptions beyond those in the statement of the theorem.
First, we note that upon adding back in the remaining non-support vector data points, the margin of the optimal separating hyperplane cannot get any wider (since each extra data point is an extra constraint).
By the paragraph above, a sufficiently small perturbation of the support vectors produces a new vector $\mathbf{w}$ that is currently known to be optimal \emph{from the perspective of the support vectors alone}, maintaining the property that the perturbations of the support vectors remain on the margins of this hyperplane.
So, it suffices to show that when \emph{all} of the data points are considered, each perturbed by a sufficiently small amount, no data point that is not originally a support vector either comes into contact with or crosses to the other side of either margin.
To see this, note that by the paragraph above, the normal vector $\mathbf{w}$ is given by a rational equation in terms of the variables which are the entries of the vectors $\mathbf{u}_i$ and $\mathbf{v}_j$.
By the continuity of this rational expression, sufficiently small perturbations of the support vectors $\mathbf{u}_i$ and $\mathbf{v}_j$ will move the endpoints and the direction of the vector $\mathbf{w}$ by an arbitrarily small amount.
Therefore, a sufficiently small perturbation of \emph{all} of the data points will not move a data point that was not originally a support vector to either intersect or cross one of the margins $H_+$ or $H_-$.
It follows that the identities of the support vectors are maintained under sufficiently small perturbations.
\end{proof}

Now we can show that points in strong general position have sufficient restrictions such that perturbing all points by a small amount does not change the points identified as support vectors.

\begin{theorem}
\label{thm:perturbation}
Given a set $X$ of linearly separable points in strong general position, there exists a $\delta>0$ such that simultaneously perturbing every point in $X$ by at most $\delta$ does not change which points are identified as support vectors.
\end{theorem}

\begin{proof}
It suffices to show that the hypotheses $(a)$--$(b)$ of Theorem~\ref{thm:small_perturb} are satisfied, when $X$ is a set of linearly separable points in strong general position.
Note that $(a)$ is satisfied by the definition of strong general position.
The fact that there is a unique Radon point follows from Lemma~\ref{lem:svm-hulls-intersect-pt}.
Furthermore, in the formulas,
\[\mathbf{r}_+=\sum_{i=0}^k s_i\mathbf{u}_i
 \quad\text{and}\quad \mathbf{r}_-=\sum_{j=0}^{\ell} t_j\mathbf{v}_j,
\]
expressing the projected Radon points $\mathbf{r}_+$ and $\mathbf{r}_-$ in terms of the support vectors $\mathbf{u}_i$ and $\mathbf{v}_j$,
each $s_i$ and $t_j$ must be positive, else we could omit a support vector and get the same separating hyperplane, contradicting condition $(iii)$ in the definition of strong general position.
This gives $(b)$.
Thus, strong general position satisfies the hypotheses for Theorem~\ref{thm:small_perturb}, and hence, a sufficiently small perturbation does not change which points are identified as support vectors.
\end{proof}

We furthermore show that for linearly separable points $X \subseteq \mathbb{R}^n$ in strong general position, there are at most $n+1$ support vectors.

\begin{theorem}\label{thm:at-most-n+1}
Suppose $X\subseteq \mathbb{R}^n$ is in strong general position, and that $X$ is equipped with linearly separable labels.
Then there are at most $n+1$ support vectors.
\end{theorem}

\begin{proof}
Suppose for a contradiction that we have at least $n+2$ support vectors.
We proceed with two cases.

First, suppose we have only one positive support vector, and therefore at least $n+1$ negative support vectors. 
Then, we have at least $n+1$ points in an $(n-1)$-flat, which is more than required to determine an $(n-1)$-flat.
This violates $(i)$ in Definition~\ref{def:sgp2}.

Now suppose we have $k+1\ge 2$ positive support vectors and $\ell+1\ge 2$ negative support vectors with $(k+1)+(\ell+1)\ge n+2$, i.e.\ with $k+\ell\ge n$.
By Definition~\ref{def:sgp2}$(i)$, the affine span of the positive support vectors is a $k$-flat, and the affine span of the negative support vectors is an $\ell$-flat.
Projecting the positive and negative support vectors onto the separating $(n-1)$-dimensional hyperplane produces a $k$-flat and an $\ell$-flat, respectively, living in an $(n-1)$-flat.
Indeed, to see (for example) that the dimension $k$ does not decrease upon projecting, note that the positive support vectors live in a hyperplane parallel to the separating hyperplane.
Since $k+\ell>n-1$, the projected $k$- and $\ell$-flats intersect in at least a line inside the separating hyperplane.
This means that prior to projecting, the $k$- and $\ell$-flats of the positive and negative support vectors contain parallel vectors, contradicting Definition~\ref{def:sgp2}$(ii)$.

Therefore we can have at most $n+1$ support vectors.
\end{proof}

Additionally, we show that we can have anywhere from 2 to $n+1$ support vectors for points in strong general position.

\begin{theorem}
\label{thm:prescribed-number-support-vectors}
Let $i,j \geq 1$ with $i+j\leq n+1$. 
Then there is a labeled subset $X\subseteq\mathbb{R}^n$ in strong general position with $i$ support vectors from the positive class and with $j$ support vectors from the negative class.
\end{theorem}

\begin{proof} 
Suppose $i$, $j$, and $n$ are as above.
Let $\{\mathbf{v}_m\}_{m=1}^i$ be the vertices of a regular $(i-1)$-simplex $I$ in $\mathbb{R}^{i-1}$, and let $\{\mathbf{w}_n\}_{n=1}^j$ be vertices of a regular $(j-1)$-simplex $J$ in $\mathbb{R}^{j-1}$.
If necessary, translate $I$ and $J$ so that the interiors of both convex hulls contain the origin.
Define the sets
\begin{align*}
X_+ &=\{(\mathbf{v}_1,\mathbf{0}_{j-1},1), (\mathbf{v}_2,\mathbf{0}_{j-1},1), \ldots, (\mathbf{v}_i,\mathbf{0}_{j-1},1)\} \subseteq \mathbb{R}^{i+j-1} \\
X_-&=\{(\mathbf{0}_{i-1},\mathbf{w}_1,-1), (\mathbf{0}_{i-1},\mathbf{w}_2,-1), \ldots, (\mathbf{0}_{i-1},\mathbf{w}_j,-1)\} \subseteq \mathbb{R}^{i+j-1},
\end{align*}
where the length of each zero vector is denoted by the associated subscript, so that $X_+$ and $X_-$ are subsets of $\mathbb{R}^{i+j-1}$.
Note that since $i+j-1\leq n,$ these points can naturally be embedded in $\mathbb{R}^n$.

We claim that the sets $X_+$ and $X_-$ are linearly separable, with the optimal separating hyperplane $H$ defined by the equation $x_{i+j-1}=0$.
First, note that the parallel hyperplane margins $H_+$ and $H_-$ are defined by the equations $x_{i+j-1}=1$ and $x_{i+j-1}=-1$, respectively.
Now consider the intersection $\rho(\text{conv}(X_+))\cap\rho(\text{conv}(X_-))$ of the orthogonal projections of the respective convex hulls onto the hyperplane $H$.
Certainly, $\mathbf{0}\in\mathbb{R}^n$ is in this intersection.
Furthermore, suppose $\mathbf{p}\in\mathbb{R}^n$ is in this intersection.
Since $\mathbf{p}\in \rho(\text{conv}(X_+)),$ there exist coefficients $\alpha_k$ such that $\mathbf{p}=\left(\sum_k \alpha_k\mathbf{v}_k,\mathbf{0}_{j-1},0\right),$ with $\alpha_k\geq 0$ and $\sum_k\alpha_k=1$.
Similarly, since $\mathbf{p}\in \rho(\text{conv}(X_-)),$ $\mathbf{p}$ is of the form $\mathbf{p}=\left(\mathbf{0}_{i-1},\sum_\ell \beta_\ell\mathbf{w}_\ell,0\right),$ where $\beta_\ell\geq 0$ and $\sum_\ell\beta_\ell=1$.
It follows that $\mathbf{p}=\mathbf{0},$ and thus  $\rho(\text{conv}(X_+))\cap\rho(\text{conv}(X_-))=\left\{\mathbf{0}\right\}$.
By Theorem \ref{thm:svm-hulls-intersect}, $H$ is the optimal separating hyperplane.

Furthermore, each point in $X_+$ and $X_-$ is a support vector: the hyperplanes $H_+$ and $H_-$ are the hyperplanes parallel to the optimal separating hyperplane $H$ that achieve the greatest margin of separation, and these planes fully contain $X_+$ and $X_-.$

It remains to handle the strong general position criterion.
The construction of $X=X_+\cup X_-$ satisfies the hypothesis of Theorem~\ref{thm:small_perturb}, and hence the identities and the number of support vectors are stable under small perturbations. 
Since strong general position is satisfied on a Zariski open set (Theorem~\ref{thm:zariski}), if $X$ is not already in strong general position, then there exists a small perturbation such that the perturbed dataset is in strong general position and has the prescribed number of support vectors in each class. 
Therefore, there is a labeled subset in $\mathbb{R}^n$ in strong general position with $i$ support vectors from the positive class and with $j$ from the negative class.
\end{proof}

We remark that one could alternatively prove Theorem~\ref{thm:prescribed-number-support-vectors} by embedding a regular $(i+j-1)$-dimensional simplex in $\mathbb{R}^n$, and labelling $i$ vertices to be in the positive class and the remaining $j$ to be in the negative class.
We show these points are in strong general position in Appendix~\ref{app:constructive}.

%%%%%%%%%%%%% Section %%%%%%%%%%%%%%

\section{Computational experiments on the type of Radon configurations}
\label{sec:experiments}

We provide additional intuition behind the interaction between Radon's theorem and support vector machines via computational experiments.
We wrote a Python program to compute the distribution of support vector configurations for randomly sampled, linearly separable data.
For each trial, we sample 10 points from two normal distributions with standard deviations 1.
The centers of these two normal distributions are drawn uniformly at random from the square $[-a,a]\times[-a,a]$ in $\mathbb{R}^2$ or cube $[-a,a]\times[-a,a]\times[-a,a]$ in $\mathbb{R}^3$, where we vary $a$ over $a=5$, $a=10$, and $a=20$.
The 10 points from one normal distribution are in the positive class, and the 10 points from the second distribution are in the negative class.
If the two classes of points are not linearly separable, then we discard the trial.
(In other words, our randomly sampled points are conditioned to be linearly separable.)
We compute the separating hyperplane, and then classify the type of Radon configuration. 
The configurations are put into categories: two support vectors, three support vectors, or four support vectors (in $\mathbb{R}^3$).

\begin{figure}[htb]
\centering
\begin{subfigure}{.5\textwidth}
  \centering
  \includegraphics[width=.8\linewidth]{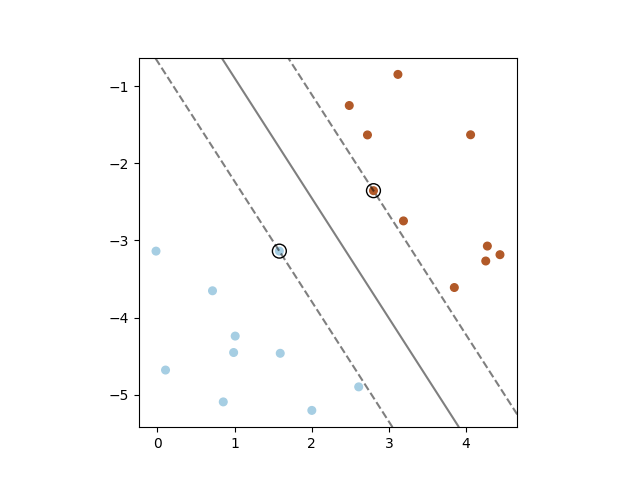}
  \caption{Two support vectors}
  \label{fig:2svs}
\end{subfigure}%
\begin{subfigure}{.5\textwidth}
  \centering
  \includegraphics[width=.8\linewidth]{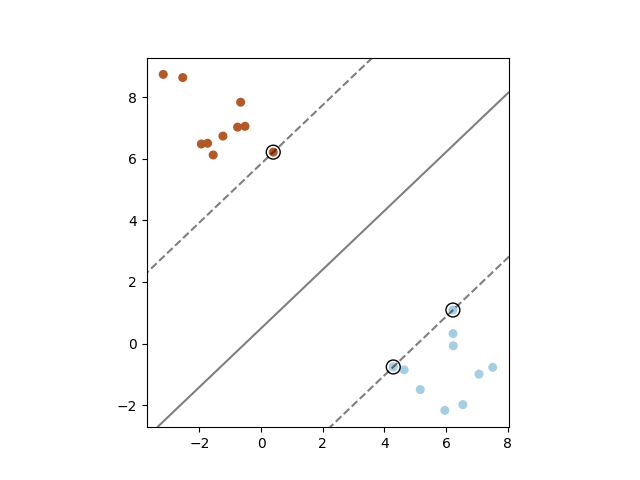}
  \caption{Three support vectors}
  \label{fig:3svs}
\end{subfigure}
\caption{Example images of randomly generated linearly seperable groups of points with (A) two and (B) three support vectors, both drawn from trials where $a=10$.}
\label{fig:svsfor2d}
\end{figure}

\begin{table}[htb]
\begin{center}
\begin{tabular}{|r||c|c|c|} \hline
$a = $ & 5 & 10 & 20 \\ \hline\hline
\# configurations with 2 support vectors & 417 & 632 & 809 \\ \hline
\# configurations with 3 support vectors & 583 & 368 & 191 \\ \hline
\end{tabular}
\caption{The number of support vectors in $\mathbb{R}^2$ from 1,000 trials, as a function of a parameter $a$ controlling the expected distance between the center of the distributions.}
\label{table:simulations2}
\end{center}
\end{table}

\begin{table}[htb]
\begin{center}
\begin{tabular}{|r||c|c|c|} \hline
$a = $ & 5 & 10 & 20 \\ \hline\hline
\# configurations with 2 support vectors & 279 & 554 & 758 \\ \hline
\# configurations with 3 support vectors & 458 & 367 & 221 \\ \hline
\# configurations with 4 support vectors (2--2 split) & 145 & 47 & 9 \\ \hline
\# configurations with 4 support vectors (3--1 split) & 118 & 32 & 12 \\ \hline
\end{tabular}
\caption{The number of support vectors in $\mathbb{R}^3$ from 1,000 trials, as a function of a parameter $a$ controlling the expected distance between the center of the distributions.}
\label{table:simulations3}
\end{center}
\end{table}

For the case of $\mathbb{R}^2$, see Table~\ref{table:simulations2}.
Out of 1,000 trials, for $a=5$, we had 417 cases with two support vectors and 583 cases with three support vectors.
As we increase the size of the box (and hence also the expected distance between centers) to $a=10$, we get 632 cases of two support vectors and 368 cases of three support vectors. 
For the largest $a$ value, $a=20$, we had 809 cases of two support vectors and 191 cases of three support vectors.
As we increase the expected distance between centers, the likelihood of having two support vectors increases, while the complimentary likelihood of having three support vectors decreases. 

In $\mathbb{R}^3$, we have similar results (see Table~\ref{table:simulations3}).
An interesting difference is that there are now two different types of configurations with 4 support vectors: a 2--2 split with 2 support vectors in both the positive and negative classes, versus a 3--1 split with 3 support vectors in one class and 1 in the other.
Out of 1,000 trials, for $a=5$, we had 279 cases of two support vectors, 458 cases of three support vectors, and 263 cases of four support vectors. 
As we increased the size of the bounding box to $a=10$, we have 554 cases with two support vectors, 367 cases of three support vectors, and 79 cases of four support vectors. 
Finally, with $a=20$, we have 758 cases with two support vectors, 221 cases of three support vectors, and 21 cases of four support vectors. 
Just as we observed in $\mathbb{R}^2$, it is also the case in $\mathbb{R}^3$ that as we increase the expected distance between centers, the likelihood of having two support vectors increases, while the likelihood of having three or four support vectors decreases.

Our Python code is available at \url{https://github.com/brimcarr/svm_radon}.

%%%%%%%%%%%%%%%%%%%%%%%%%%%%%%%%%%%%%%%%%%%%%%%%%%%

\section{Conclusion}
In this paper, we examined the interaction between support vector machines and Radon's theorem. 
After establishing background on SVM, Radon's theorem, and some additional algebraic tools in Sections~\ref{sec:svmbackground} and~\ref{sec:radon}, we report on support vectors and their possible Radon configurations in Sections~\ref{sec:RadonsThmandSVM} and~\ref{sec:svm-general-selected}. 
Given a set of linearly separable data, the projection of the convex hulls of the positive and negative support vectors to the separating hyperplane intersect if and only if the separating hyperplane is optimal.
Furthermore, if the points are in strong general position, then there is a unique point of intersection, the Radon point, if and only if the separating hyperplane is optimal.
For points in strong general position, there can be at most $n+1$ support vectors in $\mathbb{R}^n$, and every number of support vectors from $2$ to $n+1$ can be attained.
Finally, we show that given linearly separable data in strong general position, there exists a $\delta>0$ such that perturbing the data points by at most $\delta$ preserves the data points labeled as support vectors.

Our research shows that for points in strong general position there is a maximum number of support vectors, and that we are guaranteed to get one of a certain number of Radon configurations. 
If a computation returns too many support vectors, this means that the data is not in strong general position, implying that small perturbations could change the set of support vectors.
Alternatively, the code computing the support vectors needs to be run at an increased level of precision.
We found this to be the case when running standard MATLAB and python algorithms for hard margin support vector machines: for large numbers of points, the code could return more than $n+1$ support vectors in $\mathbb{R}^n$, but those additional support vectors disappeared upon increasing the level of machine precision.

We share a list of questions that are natural to ask following our results.

\begin{question}
What can be said for spherical or ellipsoidal SVM? See for example~\cite{yao2008utilizing}, especially Figure~3 within.
In hard margin spherical SVM, we suppose that the two classes of data can be separated with one class inside a sphere, and with the second class outside.
Ellipsoidal SVM is analogous, except that the sphere is generalized to also allow for separating ellipsoids.
What is the maximal number of possible support vectors for separable points in general position for spherical and ellipsoidal SVM, where ``general position" will have to be re-defined for this new context?
Is there a version of Radon's theorem that relates to spherical or ellipsoidal SVM?
\end{question}

\begin{remark}
It is of course possible that two classes of data are not separable by a hyperplane.
Nevertheless, they may still be nicely separable by some nonlinear $(n-1)$-dimensional surface in $\mathbb{R}^n$.
When this is the case, one can use the \emph{kernel trick} to nonlinearly map the data from $\mathbb{R}^n$ into $\mathbb{R}^N$ with $N>n$, and then find a separating hyperplane in $\mathbb{R}^N$. 
One often searches over several such maps in order to find one enabling classification. 
The kernel trick allows one to compute in $\mathbb{R}^N$ nearly as efficiently as in $\mathbb{R}^n$.
For a more in-depth look at the kernel trick, please see~\cite{cristianini2000SVM}.
With kernel SVM, we can still say something about the support vectors: when the embedded data is linearly separable in and in strong general position in $\mathbb{R}^N$, then there can be at most $N+1$ support vectors.
\end{remark}

\begin{question}
Is there anything precise that can be said when the data is not linearly separable, and hence one uses soft-margin SVM (allowing errors) instead of hard margins?
There are several different ways one could define support vectors in this context.
\end{question}

\begin{question}
\label{ques:tverberg}
Tverberg's theorem~\cite{barany2018tverberg,tverberg1966generalization} states that given $(r-1)(n+1)+1$ points in $\mathbb{R}^n$, there is a partition into $r$ parts whose convex hulls intersect. 
Tverberg's theorem is a generalization of Radon's theorem, which is obtained by restricting to $r=2$.
Is Tverberg's theorem related to any of the several versions of multiclass SVMs~\cite{duan2005best,hsu2002comparison}?
\end{question}

\begin{question}
Which of the possible Radon configurations (say one support vector in each class, or one support vector in the positive class and two in the negative class, or two support vectors in each class, \textit{etc.}) is the most likely?
For this question to make sense, we need a random model of labeled data points that are linearly separable.
What are reasonable such models, beyond the simple model considered in Section~\ref{sec:experiments}?
There is almost certainly no canonical such model, but instead a zoo of random models of linearly separable data that one could consider.

One class of probability models could be as follows. Select two different probability distributions in $\mathbb{R}^n$ with linearly separable supports.
Sample positively labeled points from one distribution, and negatively labeled points from the second.
The resulting points will be linearly separable.

A second class of probability models could be to consider two arbitrary probability distributions in $\mathbb{R}^n$, one corresponding to each of the positive and negative classes.
Upon sampling a point from a distribution, if that new point makes it so that the data are no longer linearly separable, then reject that point and sample again at random.

Under any such random model for linearly separable labeled points, we are interested in the following question.
Which Radon configurations occur with positive probability, and what are those probabilities?
The answers will of course depend on the random model selected.

The paper~\cite{finkseparability} %(\footnote{\url{https://www.cs.ucsb.edu/~suri/psdir/euro.pdf}})
considers the problem of computing the probability that two probabilistic point sets are linearly separable.
\end{question}

\begin{question}
If you toss two 6-sided die into the air and they collide, then (with probability one) they either collide vertex-to-face or edge-to-edge.
The probability of a vertex-to-face collision is 46\%, and the probability of an edge-to-edge collision 54\%.
The answer is computed using integral geometry.
This is called Firey's dice problem; see~\cite{firey1972integral,firey1974kinematic,mcmullen1974dice} and~\cite[Pages 358--359]{schneider2008stochastic}. %(\footnote{\url{http://www.math.utah.edu/~treiberg/IntGeomSlides.pdf}})

When the two dice collide, it is analogous to a very special case of SVM where the support vectors from both classes all lie on the same hyperplane, i.e.\ the width of the margin is zero.
A question that is perhaps more closely related to SVM is: what is the probability of the different types of ``closest point configurations" when we specify that the dice are at exactly distance $d$ apart? 
Once $d>0$, the two closest points can both be vertices with positive probability, or the two closest points can be a vertex and a point on an edge. The vertex-and-face and edge-and-edge configurations still occur with positive probability.
The Firey dice problem can perhaps be considered as a limit as $d\to0$.
This problem is also interesting for dice of any convex shape, for example tetrahedral dice.
\end{question}

%There are many new directions to extend this research and we hope to pursue them at a later date.

%%%%%%%%%%%%%%%%%%%%%%%%%%%%%%%%%%%%%%%%%%%%%%%%%%%

\section*{Acknowledgements}

We would like to thank Michael Kirby, Chris Peterson, and Simon Rubinstein--Salzedo for helpful conversations.
We would also like thank the anonymous reviewers for their significant contributions.
In particular, we would like to thank an anonymous reviewer for suggesting and proving the converse result in Theorem~\ref{thm:svm-hulls-intersect} as well as providing us with the tools to help us concisely prove Theorems~\ref{thm:axiom_equiv}, \ref{thm:zariski}, \ref{thm:small_perturb}, and \ref{thm:perturbation}.

%MK - ellipsoidal SVM
%CP - Kolmogorov 0-1 law
%SRS - Firey dice problem

\bibliographystyle{plain}
\bibliography{svmRadon.bib}

\appendix

\section{An example of a direct proof of strong general position}
\label{app:constructive}

The purpose of this appendix is to provide a concrete example of a set of points that satisfy the conditions of strong general position, along with a direct proof that strong general position holds. 
We use simplicies as the basis for our example, as they are easy to visualize and show up in a wide range of applications.

Before we begin, recall that a $(k-1)$-simplex, $\Delta^{k-1}$, can be embedded in the $(k-1)$-dimensional linear subspace of $\mathbb{R}^{k}$ satisfying the constraint $x_1+\ldots+x_k=1$.
Indeed, note that $\Delta^{k-1}$ is the convex hull of the standard basis vectors $\mathbf{e}_1=(1,0,\ldots,0)$, \ldots, $\mathbf{e}_k=(0,\ldots,0,1)$ in a $(k-1)$-dimensional subspace of $\mathbb{R}^k$, which we will think of as a copy of $\mathbb{R}^{k-1}$ in $\mathbb{R}^k$.
For example, we can view the 2-simplex $\Delta^2$ as the convex hull of the standard basis vectors $\mathbf{e}_1$, $\mathbf{e}_2$, and $\mathbf{e}_3$ inside the 2-plane $x_1+x_2+x_3=1$ in $\mathbb{R}^3$ (see Figure~\ref{fig:2_simplex_in_R3}).
This will be a convenient approach in the proof below.

\begin{figure}[h]
    \centering
\begin{tikzpicture}[scale=0.6]
    \coordinate [] (A) at (0,3);
    \coordinate [] (B) at (3,0);
    \coordinate [] (C) at (-1.5,-1);

    \foreach \i in {A,B,C}
    \draw[-, fill=gray!30, opacity=.5] (A)--(B)--(C)--cycle;
    
    \draw[->,thick] (0,0)--(0,4);
    \draw[->,thick] (0,0)--(4,0);
    \draw[->,thick] (0,0)--(-2,-1.3333);
    
    \node [circle,fill=blue,inner sep=0pt,minimum size=4pt,label=right:{$(0,0,1)$}] (G) at (0,3) {};
    \node [rectangle,fill=red,inner sep=0pt,minimum size=4pt,label=above left:{$(1,0,0)$}] (G) at (-1.5,-1) {};
    \node [rectangle,fill=red,inner sep=0pt,minimum size=4pt,label=above right:{$(0,1,0)$}] (G) at (3,0) {};

    \end{tikzpicture}\\
    \caption{A 2-simplex embedded in the 2-dimensional subspace $x_1+x_2+x_3=1$ of $\mathbb{R}^3$ with support vectors assigned to the vertices.}
    \label{fig:2_simplex_in_R3}
\end{figure}
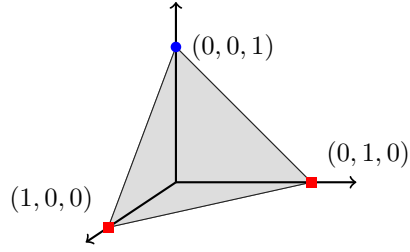

Define the set $X$ to be the vertices of the $(m-1)$-simplex that is defined as above.
That is, $X$ is contained in an $(m-1)$-dimensional subspace (i.e.\ a copy of $\mathbb{R}^{m-1}$) in $\mathbb{R}^{m}$, and the elements of $X$ are the standard basis vectors $\{\mathbf{e}_k\}_{k=1}^{m}$.
Note that $X$, as we've defined it, exists in $\mathbb{R}^{m-1};$ since $m-1\leq n$, $X$ can naturally be embedded in $\mathbb{R}^n$.
We claim that $X$ is in strong general position.

\subsection*{Strong general position condition (i)}
Let $k$ be an integer with $0\leq k\leq n-1$. Suppose $F$ is a $k$-flat generated by points of $X$, $\{\mathbf{e}_\alpha\}_{\alpha\in\mathcal{A}}$. 
Any point in $F$ is of the form $\sum_{\alpha\in\mathcal{A}}c_\alpha\mathbf{e}_\alpha$, where $\sum_{\alpha\in\mathcal{A}}c_\alpha=1$.
Since the other elements of $X$ are each of the form $\mathbf{e}_\beta$ for some $\beta\notin\mathcal{A}$, no other element of $X$ is an element of $F$.

\subsection*{Strong general position condition (ii)}
Suppose $F_A$ and $F_B$ are disjoint flats generated by points in $X$:
suppose $F_A$ is generated by $\{\mathbf{e}_\alpha\}_{\alpha\in\mathcal{A}}$ and $F_B$ is generated by $\{\mathbf{e}_\beta\}_{\beta\in\mathcal{B}}$.
Note that $\mathcal{A}\cap\mathcal{B}=\emptyset$, since the flats are disjoint.
Suppose $\mathbf{a}$ and $\mathbf{b}$ are non-zero displacement vectors contained in $F_A$ and $F_B$, respectively.
Then there exist coefficients $\{\ell_\alpha\}_{\alpha\in\mathcal{A}}, \{m_\alpha\}_{\alpha\in\mathcal{A}}, \{p_\beta\}_{\beta\in\mathcal{B}}, \{q_\beta\}_{\beta\in\mathcal{B}}$ with $\sum_{\alpha\in\mathcal{A}}\ell_\alpha=\sum_{\alpha\in\mathcal{A}}m_\alpha=\sum_{\beta\in\mathcal{B}}p_\beta=\sum_{\beta\in\mathcal{B}}q_\beta=1$ and
\begin{align*}
\mathbf{a}
&=\sum_{\alpha\in\mathcal{A}}\ell_\alpha\mathbf{e}_\alpha-\sum_{\alpha\in\mathcal{A}}m_\alpha\mathbf{e}_\alpha
=\sum_{\alpha\in\mathcal{A}}(\ell_\alpha-m_\alpha)\mathbf{e}_\alpha \text{ and } \\ \mathbf{b}
&=\sum_{\beta\in\mathcal{B}}p_\beta\mathbf{e}_\beta-\sum_{\beta\in\mathcal{B}}q_\beta\mathbf{e}_\beta
=\sum_{\beta\in\mathcal{B}}(p_\beta-q_\beta)\mathbf{e}_\beta.
\end{align*}
Let $\lambda$ be a real number and assume $\mathbf{a}=\lambda\mathbf{b}$.
Then we have $\lambda=0$, since $\mathcal{A}\cap\mathcal{B}=\emptyset$.
So there are no parallel vectors in $F_A$ and $F_B$.

\subsection*{Strong general position condition (iii)}
Consider two disjoint flats $F_A$ and $F_B$ generated by points in $X$:
suppose $F_A$ is generated by $\{\mathbf{e}_\alpha\}_{\alpha\in\mathcal{A}}$ and $F_B$ is generated by $\{\mathbf{e}_\beta\}_{\beta\in\mathcal{B}}$, where $\mathcal{A}\cap\mathcal{B}=\emptyset$.
In order to prove that $X$ satisfies condition (iii) of strong general position, we need to identify a vector $\mathbf{w}$ that points from $F_A$ to $F_B$ whose length is equal to the distance between $F_A$ and $F_B$.
We claim that the vector $\mathbf{w}$ is given by the difference between the centroids of the points and we show that $\mathbf{w}$ does have the appropriate length by showing that $\mathbf{w}$ is orthogonal to all displacement vectors in $F_A$ and $F_B$, respectively. 

Suppose $|\mathcal{A}|=a$ and $|\mathcal{B}|=b$, and define the centroids $\mathbf{c}_A=\frac{1}{a}\sum_{\alpha\in\mathcal{A}}\mathbf{e}_\alpha$ and $\mathbf{c}_B=\frac{1}{b}\sum_{\beta\in\mathcal{B}}\mathbf{e}_\beta$.
Then $\mathbf{w}=\mathbf{c}_B-\mathbf{c}_A=\frac{1}{b}\sum_{\beta\in\mathcal{B}}\mathbf{e}_\beta - \frac{1}{a}\sum_{\alpha\in\mathcal{A}}\mathbf{e}_\alpha$.
Suppose $\mathbf{v}$ is a displacement vector in $F_A$.
Then $\mathbf{v}$ is a difference of two affine combinations of the generators of $F_A$, so there exist coefficients $\{\ell_\alpha\}_{\alpha\in\mathcal{A}}$ and $\{m_\alpha\}_{\alpha\in\mathcal{A}}$
such that $\sum_{\alpha\in\mathcal{A}}\ell_\alpha=\sum_{\alpha\in\mathcal{A}}m_\alpha=1$ and $\mathbf{v}=\sum_{\alpha\in\mathcal{A}}\ell_\alpha\mathbf{e}_\alpha - \sum_{\alpha\in\mathcal{A}}m_\alpha\mathbf{e}_\alpha=\sum_{\alpha\in\mathcal{A}}(\ell_\alpha-m_\alpha)\mathbf{e}_\alpha$.
We compute the dot product of $\mathbf{v}$ and $\mathbf{w}$, obtaining the following:

\begin{align*}
    \mathbf{v}\cdot\mathbf{w} &= \mathbf{v}\cdot\mathbf{c}_B - \mathbf{v}\cdot\mathbf{c}_A \\
    &= -\mathbf{v}\cdot\mathbf{c}_A &\text{ since } \mathbf{e}_\alpha\cdot\mathbf{e}_\beta=0 \text{ for all } \alpha\in\mathcal{A}, \beta\in\mathcal{B}\\
    &= -\left(\sum_{\alpha\in\mathcal{A}}(\ell_\alpha-m_\alpha)\mathbf{e}_\alpha\right) \cdot \frac{1}{a}\sum_{\alpha\in\mathcal{A}}\mathbf{e}_\alpha\\
    &= -\frac{1}{a}\sum_{\alpha_1\in\mathcal{A}}\sum_{\alpha_2\in\mathcal{A}}(\ell_{\alpha_1}-m_{\alpha_1})\mathbf{e}_{\alpha_1}\cdot\mathbf{e}_{\alpha_2}\\
    &= -\frac{1}{a}\sum_{\alpha_1\in\mathcal{A}}\sum_{\alpha_2\in\mathcal{A}}(\ell_{\alpha_1}-m_{\alpha_1})\delta_{\alpha_1,\alpha_2}\\
    &= -\frac{1}{a}\sum_{\alpha_1\in\mathcal{A}}(\ell_{\alpha_1}-m_{\alpha_1})\\
    &= -\frac{1}{a}(1-1)\\
    &= 0.
\end{align*}
So $\mathbf{w}$ is orthogonal to $F_A$.
The verification that $\mathbf{w}$ is orthogonal to $F_B$ is similar.

Thus $\mathbf{w}$ is the vector pointing from $F_A$ to $F_B$ with length equal to the distance between the hyperplanes.
The hyperplane orthogonal to $\mathbf{w}$ containing $F_A$ is then given by $\mathcal{P}\colon\mathbf{w}\cdot\mathbf{x}+\frac{1}{a}=0$.
We need to show that $\mathcal{P}$ contains no other points of $X$.
Let $Y$ be the points in $X$ that are not generators of $F_A$ and $F_B$.
Define $\mathcal{K}=\{1,2,\ldots,m\}-(\mathcal{A}\cup\mathcal{B})$; then $Y=\{\mathbf{e}_k\}_{k\in\mathcal{K}}$.
We will show that each $\mathbf{y}\in Y$ fails to satisfy the equation for the hyperplane $\mathcal{P}$.
Choose an arbitrary $\mathbf{y}\in Y$ such that $\mathbf{y}=\mathbf{e}_k$ for some $k\in\mathcal{K}$.
We have
\[
    \mathbf{w}\cdot\mathbf{y} + \frac{1}{a} = \left( \frac{1}{b}\sum_{\beta\in\mathcal{B}}\mathbf{e}_\beta - \frac{1}{a}\sum_{\alpha\in\mathcal{A}}\mathbf{e}_\alpha \right) \cdot \mathbf{e}_k + \frac{1}{a}
    =\frac{1}{a}
    \neq 0,
\]
since $\mathbf{e}_k$ is orthogonal to each $\mathbf{e}_\alpha$ and each $\mathbf{e}_\beta$.
Thus $\mathcal{P}$ contains no other points of $X$, and condition (iii) of strong general position is satisfied.\\

\end{document}